\documentclass{article}

\PassOptionsToPackage{numbers, compress}{natbib}

\usepackage[preprint]{neurips_2023}

\usepackage[utf8]{inputenc} 
\usepackage[T1]{fontenc}    
\usepackage{hyperref}       
\usepackage{url}            
\usepackage{booktabs}       
\usepackage{amsfonts}       
\usepackage{nicefrac}       
\usepackage{microtype}      
\usepackage{xcolor}         

\usepackage{textcomp, gensymb}
\usepackage{graphicx}
\usepackage{caption}
\usepackage{comment}
\usepackage{subcaption}
\usepackage{wrapfig}
\usepackage{siunitx}
\usepackage{sidecap}
\usepackage[ruled,vlined]{algorithm2e}
\usepackage[page,header]{appendix}
\usepackage{titletoc}
\usepackage{adjustbox}
\usepackage{amsmath}
\usepackage{amsthm}

\DeclareMathOperator{\argmax}{argmax} 
\DeclareMathOperator{\argmin}{argmin}

\newtheorem{theorem}{Theorem}

\title{A Sparse Quantized Hopfield Network for \\Online-Continual Memory}

\author{%
  Nick Alonso\thanks{Author Contact: nalonso2@uci.edu. \\This work was supported by the National Science Foundation Grant IIS-1813785 and the Air Force Office of Scientific Research Grant FA9550-19-1-0306.} \\
  Cognitive Sciences Dept.\\
  University of California, Irvine\\
  \And
  Jeff Krichmar \\
  Cognitive Sciences Dept.\\
  Computer Science Dept.\\
  University of California, Irvine\\
}

\begin{document}
\maketitle
\begin{abstract}
An important difference between brains and deep neural networks is the way they learn. Nervous systems learn online where a stream of noisy data points are presented in a non-independent, identically distributed (non-i.i.d.) way. Further, synaptic plasticity in the brain depends only on information local to synapses. Deep networks, on the other hand, typically use non-local learning algorithms and are trained in an offline, non-noisy, i.i.d. setting. Understanding how neural networks learn under the same constraints as the brain is an open problem for neuroscience and neuromorphic computing. A standard approach to this problem has yet to be established. In this paper, we propose that discrete graphical models that learn via an online maximum a posteriori learning algorithm could provide such an approach. We implement this kind of model in a novel neural network called the Sparse Quantized Hopfield Network (SQHN). We show that SQHNs outperform state-of-the-art neural networks on associative memory tasks, outperform these models in online, non-i.i.d. settings, learn efficiently with noisy inputs, and are better than baselines on a novel episodic memory task.
\end{abstract}

\section{Introduction}
A fundamental question in computational neuroscience and neuromorphic computing is the question of how to train deep neural networks using only local learning rules in the learning scenario faced by the brain, where data is noisy and presented in an online-continual fashion. Local learning rules use only information spatially and temporally adjacent to the synapse at the time of the update (e.g., pre-synaptic and post-synaptic neuron activity). Online-continual learning occurs when a stream of single data points are presented during training in non-independent and identically distributed (non-i.i.d.) fashion (e.g., several datasets are presented one dataset at a time). Brains must learn under these conditions, and neuromorphic hardware embedded in real world systems have similar constraints \cite{davies2018loihi}.

Standard approaches to deep learning have not provided a solution to this problem. The standard approach trains neural networks with stochastic gradient descent (SGD) implemented by the backpropagation algorithm (BP) \cite{rumelhart1995backpropagation}. BP is a non-local learning algorithm generally considered biologically implausible \cite{crick1989recent, stork1989backpropagation, lillicrap2020backpropagation} and is difficult to make compatible with neuromorphic hardware \cite{neftci2019surrogate, schuman2022opportunities}. Further, the standard training paradigm is offline learning, where data is mini-batched, i.i.d., non-noisy and can be passed over for multiple epochs during training. Learning in the noisy, online-continual scenario is much more difficult than the offline scenario. Unlike offline learners, online-continual learners must avoid problems like catastrophic forgetting and are pressured to deal with noise and to be more sample efficient (faster learners).

Furthermore, recent work on this problem has not approached all the aspects of the problem simultaneously. For example, although bio-plausible algorithms have been developed for deep networks, these algorithms are typically tested and developed for offline settings (e.g., \cite{o1996biologically, whittington2017approximation, scellier2017equilibrium, sacramento2018dendritic}). Although progress has been made on online-continual learning, essentially all of this work uses BP in some capacity (see \cite{khetarpal2022towards, wang2023comprehensive, parisi2020online, gallardo2021self, mai2022online, hayes2022online}). Some works do test local learning algorithms in online and continual settings separately, but do not address the online and continual setting simultaneously or focus on shallow recurrent networks (e.g., \cite{bellec2020solution, yoo2022bayespcn, yin2023accurate}). Therefore, novel local learning models that perform online-continual learning without resorting to BP are needed.

We attempt to remedy this situation by making the following contributions: 1) Unlike previous works on online-continual learning, which tend to focus on classification, we study the more general and basic task of \textit{associative memory}, i.e., the basic process of storing and retrieving corrupted and partial patterns. We believe studying this basic task could yield ideas that apply across a wide range of tasks, rather than a single narrow task, like classification. 2) We propose a general approach to online-continual associative memory, based on the idea that a sparse, quantized neural code can both deal with noisy, partial input and prevent catastrophic forgetting. Further, we propose implementing sparse, quantized neural codes using discrete graphical models that learn via algorithms similar to maximum a posteriori (MAP) learning, where MAP learning has the advantage of using local learning rules. 3) We implement this approach in a novel neural network called the sparse quantized Hopfield network (SQHN), an energy based model that optimizes a novel energy function and utilizes a novel learning algorithm that combines neuro-genesis (neuron growth) and local learning rules, both engineered specifically to yield high performance in the noisy and online-continual setting. 4) We develop two memory tasks, which are new to the recent machine learning literature on associative memory models, the noisy encoding task and an episodic memory task. 5) We run a variety of tests showing that SQHN significantly outperforms baselines on these new tasks and matches or exceeds state of the art (SoTA) on more standard associative memory tasks. 

\section{Results}

\subsection{Toward A Foundation for Local, Online-Continual Memory Models}
Our goal is to design a model that can 1) deal with noisy, partial inputs in associative memory tasks, 2) learn in a sample-efficient way that avoids catastrophic forgetting in online-continual learning scenarios, and 3) uses only local learning rules. Our proposed approach has three parts. 

First, we propose that \textit{quantization} provides a principled approach to associative recall. Quantization is the process of mapping continuous valued inputs to a finite, discrete code, which is necessarily a process of pattern completion, i.e., a process where many distinct vectors are mapped to the same vector. The associative memory problem may also be cast as one of pattern completion, where the goal is to reconstruct stored data points, $x$, given corrupted or partial versions, $\tilde{x}$, of it, i.e., $[\tilde{x}_0, \tilde{x}_1,...] \rightarrow x$. Quantization can be used to perform this mapping via $[\tilde{x}_0, \tilde{x}_1,...] \rightarrow h^* \rightarrow x$, where $h^*$ is a single discrete latent code. Pattern completion is intuitively more difficult with continuous latent codes, since these codes may vary in an infinite number of ways, making it more difficult to map many corrupted versions of a data point to the same latent code and reconstruction (e.g., see experiments and discussion). 

Second, we primarily use \textit{parameter isolation} to avoid catastrophic forgetting, which is a strategy that has recent success in BP-based deep learning models (e.g.,\cite{lee2020neural, yoon2017lifelong, mallya2018packnet, mallya2018piggyback}). This strategy allocates subsets of new and old parameters to different tasks during training, as needed. By only using and updating small subsets of parameters each iteration, models are able to drastically avoid forgetting. However, there needs to be a principled method to decide which parameters to update or add at which times.

Third, we propose using an MAP learning algorithm as a local learning algorithm, in discrete-graphical models, which naturally implement quantization and parameter isolation. MAP learning works by first performing inference over hidden variables to find their specific values, $h^*$, that maximize the posterior $P(h^* | x, \theta) \propto P(h^*, x| \theta)$. In discrete graphical models these values are integers. Parameters are then updated to further increase the probability of the joint $P(h^*, x| \theta)$. Local learning rules are naturally used in this algorithm (supp. \ref{supp:learn}). Further, if integer values are represented by sparse, one-hot vectors, updates are also sparse, i.e. perform parameter isolation. 

\subsection{The Sparse Quantized Hopfield Network}
We develop a novel implementation of a discrete graphical model that uses primarily neural network operations. We call it the \textit{sparse quantized Hopfield network} (SQHN). SQHNs have architectural similarities to Hopfield networks and Bayesian networks. Unlike standard Bayesian networks, SQHNs are relatively easy to scale and more compatible with hardware that assumes vector matrix multiplication as the basic operation (e.g., GPUs and memristors). Unlike common Hopfield nets, SQHNs explicitly utilize quantization and implement a discrete, directed graphical model. Hidden nodes in SQHN models are assigned integer values during inference, represented by sparse one-hot vectors. Sparsity distinguishes SQHNs from prior quantized Hopfield networks (e.g., \cite{matsuda1999quantized, matsuda1999theoretical}), which assign an integer value to each neuron. The sparse code we use ensures subsets of parameters are isolated during training.

\begin{figure}[t]
\includegraphics[width=.95\textwidth]{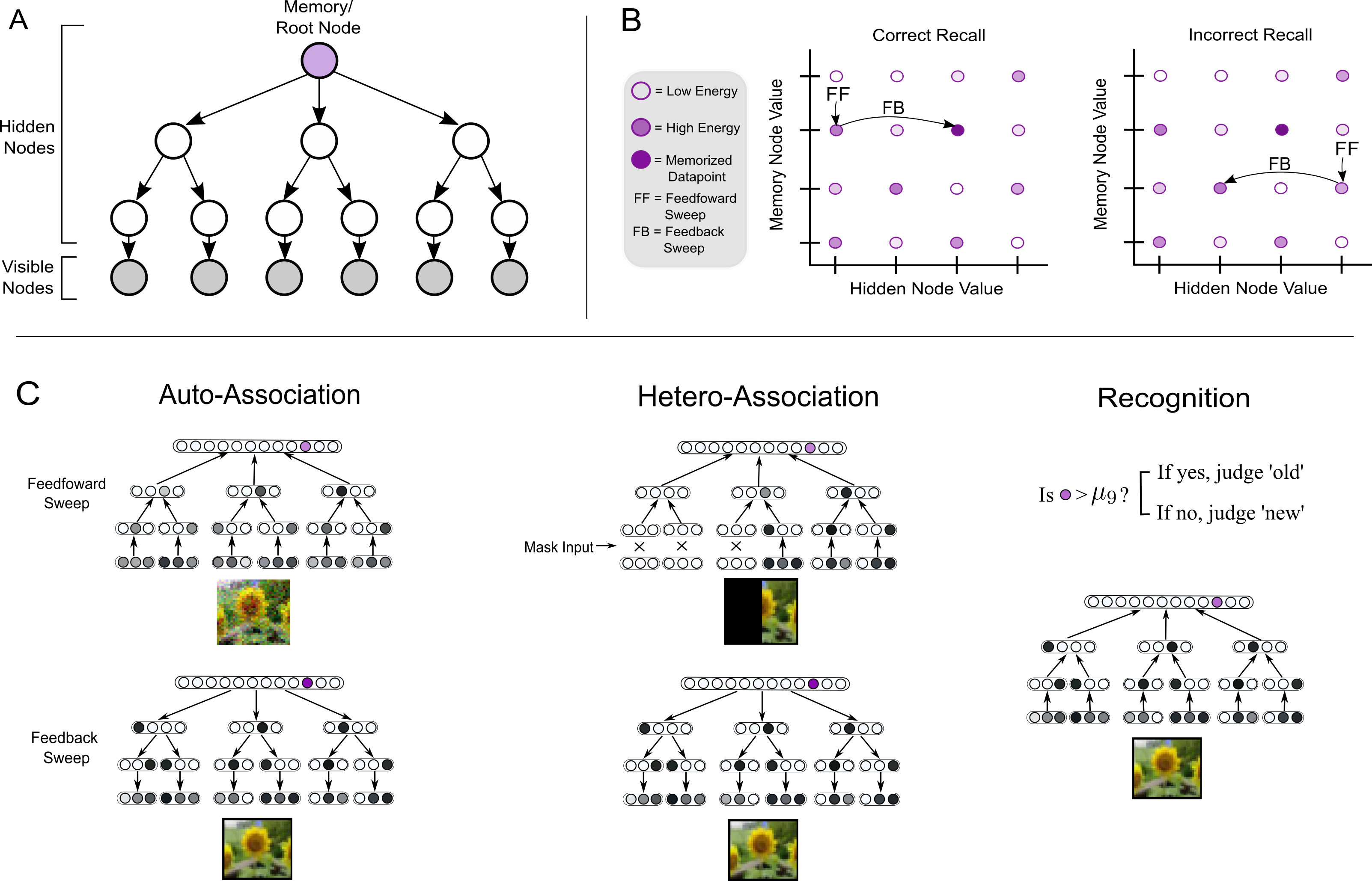}
\centering
\caption{\textbf{A.} Tree structured, directed, acyclic graph. \textbf{B.} Diagram of recall process in terms of the node values and energy. Nodes take integer values. Correct recall finds the set of node values at the global maximum of the energy. Correct recall typically only occurs if the memory node has the correct value, since the memory node is fixed after the feedforward (FF) pass, after which it then adjusts the values of hidden nodes through feed-back/top-down signals. \textbf{C.} Neural network diagrams of the SQHN during associative recall and recognition. During recall the FF sweep propagates signals up the hierarchy, where the memory/root node retrieves the most probable (high energy) value, and propagates the signals down encouraging hidden nodes to take the values associated with that particular memory/value.}
\label{fig:archDiagram}
\end{figure}

\subsubsection{Energy}
SQHNs implement direct graphical models. In this paper, we consider tree-architectures without loops (see Fig. \ref{fig:archDiagram}), though many other architectures are possible. Visible nodes are nodes clamped to portions of the input (e.g., image patches), which are assumed, though not required, to be continuous. Each hidden node $l$ represents categorical variables that take an integer value represented by a one-hot vector, $h_l^*$. We also notate clamped values at visible nodes as $h_l^*$. Conditional probability $p(h^*_l | pa_l)$ of node $l$ given its parent value is parameterized by synaptic weight matrices. For example, in the simple case where $l$ has one parent $p(h^*_l | pa_l)$ it is parameterized by matrix $M_{pa_l,l}$. The energy is a summation over conditional probabilities:
\begin{equation}\label{eq:energy}
E(h^{*}, \theta) = \frac{1}{L}\sum_{l=0}^L p(h^*_l | pa_l).
\end{equation}
Typically, in directed graphical models, like Bayesian Networks, the aim of MAP inference is to update the values of hidden nodes to maximize the joint probability of node states, which is the \textit{product} of the conditional probabilities rather than the sum \cite{bishop2006pattern}. We show that the summation of the conditional probabilities approximates a kind of joint probability that takes into account uncertainty over parameters (supp. \ref{supp:energy}). Taking into account this uncertainty is crucial for learning in online settings. Further, by approximating this joint distribution with the energy above, we can implement a model that performs inference using only standard artificial neural network operations that sum inputs to neurons rather than multiply, which would be needed with the standard joint distribution (supp. \ref{supp:infer}).

\subsubsection{Recall}
SQHN models update neuron activities during recall in a way akin to MAP inference, where neurons are updated to maximize the energy:
\begin{equation}
h^{*} = \argmax_{h^*} E(h^*, \theta, x),
\end{equation}
where $h^{*}$ is the set of one-hot vectors assigned to each node. We use a novel inference procedure, which we show empirically is a good approximate solution to this optimization problem (supp. fig. \ref{fig:energy}). This approximate inference method uses standard neural network operations and is computationally cheap, only involving a single feed-forward/bottom-up (FF) and feedback/top-down (FB) sweep through the network. Mathematically, this recall process can be understood as follows: each data point $x^t$ observed during training is stored at maxima of the energy. Therefore, given a corrupted input $\tilde{x}^t$, an SQHN can reconstruct the original by finding the latent code that maximizes the energy given $\tilde{x}^t$, which, during correct recall, will be the same latent code as the original data point (figure \ref{fig:archDiagram} B). Mechanistically, recall works by propagating signals up to the memory node which is assigned a memory value. Then a signal is propagated down, encouraging lower level hidden nodes to take values associated with that memory value (figure \ref{fig:archDiagram} C).

For a pseudo-code description, see supplementary algorithm \ref{alg:recall}.

\subsubsection{Learning}

SQHNs update their parameters using an algorithm akin to MAP learning, where each training iteration first perform inference to maximize energy w.r.t. activities and then update weights to further increase energy. 
\begin{equation}
\begin{split}
\theta^T = \argmax_\theta \sum^T_{t=0} E(\theta, h^{*,t}, x^t),
\end{split}
\end{equation}
where matrices at hidden layers must meet certain normalization constraints, making this a constrained optimization problem. Importantly, weights are updated to maximize energy over \textit{all previously observed data} points rather than just the one present at the current iteration. This helps prevent forgetting of previous data points in online scenarios. However, the update is performed using only the activities and data point from the current iteration (i.e., there is no buffer of previous data points or activities). The solution is a local Hebbian-like update rule:
\begin{equation}
\Delta M_{pa_l,l} = \frac{1}{c_{pa_l}^* + 1} (h^*_{l} - M_{pa_l, l} h^{*}_{pa_l}) h^{*\top}_{pa_l},
\end{equation}\label{eq:SQHNDeltaW}
where $c_{pa_l}^*$ is a count of the number of iterations the parent node value was activated during training. Because $h_{pa_l}^*$ is a sparse, one-hot vector, this is a sparse weight update that only alters the values in \textit{a single column} of matrix $M_{pa_l, l}$.

Importantly, instead of randomly initializing weights, we initialize all weights equal to 0, then grow new neurons and synapses as needed. That is, before weights are updated, neurons are grown during inference. A new neuron is grown at node $l$ if no neuron at the node has a value greater than some threshold. We use an exponentially decaying threshold based on the Dirichlet prior:
\begin{equation}
\epsilon = \frac{\alpha}{ (t + \alpha)},
\end{equation}
where $t$ is current training iteration and $\alpha$ is a hyper-parameter. Ablations show that if neuron growth, decaying growth threshold, or learning rate decay are removed, the model performs noticeably worse in online-continual settings (suppl. fig. \ref{fig:onlineContAbl}).

For pseudo-code description see supplementary, algorithm \ref{alg:learn}.

\subsubsection{Episodic Memory Task} 
Below, we test SQHNs on a novel binary classification task, where the model must recognize whether some particular data point was observed previously during training or not. To perform recognition, we use the value of the neuron at the root/memory node with the maximum value. This maximum value tells us how similar or probable the features of the current data point is to the features of the most similar previously observed data point (supp. \ref{supp:recog}). If the max activity at the memory node is above some threshold, the data point is judged to be old. If below, it is judged to be new. The threshold we use is the moving average, $\mu_l$, of the activity values observed for each neuron:
\begin{equation}
\mu_{l,j}^t =  \frac{1}{c^t_{l,j}} \sum_{n=1}^t h^{*,n}_{l,j} = \frac{c^t_{l,j}-1}{c^t_{l,j}}\mu_{l,j}^{t-1} + \frac{1}{c^t_{l,j}} h_{l,j}^{*,t},
\end{equation}
where the equation on the right shows how to compute this average online. We show that this value is an estimate of the desired probability of $p(x^t=old| \theta^{(t-1)}, x^t) = .5$, which is the threshold at which it becomes more probable than not that the data point is old rather than new (supp. \ref{supp:recog}).

For pseudo-code description of recognition, see supplementary algorithm \ref{alg:recogn}.

\begin{table}[t]
\centering
\begin{adjustbox}{width=1\textwidth}
\small
  \begin{tabular}{c | c c | c c | c c}
    \toprule
    \multicolumn{7}{c}{Recall MSE - Moderate Corruption}\\
    \toprule
    & \multicolumn{2}{|c|}{White Noise} & \multicolumn{2}{|c|}{Pixel Dropout} & \multicolumn{2}{|c|}{Mask} \\
    \toprule
    & CIFAR-10 & TinyImgNet & CIFAR-10 & TinyImgNet & CIFAR-10 & TinyImgNet \\
    \toprule
    GPCN(offline) \cite{yoo2022bayespcn} & $.0121^{(\pm.0001)}$ & $.0067^{(\pm.0004)}$ & $.0001^{(\pm.0000)}$ & $.0000^{(\pm.0000)}$ & $.0009^{(\pm.0000)}$ & $.0001^{(\pm.0000)}$\\
    BayesPCN \cite{yoo2022bayespcn} & $.0337^{(\pm.0007)}$ & $.6606^{(\pm.0267)}$ & $.0001^{(\pm.0000)}$ & $.0000^{(\pm.0000)}$ & $.0019^{(\pm.0000)}$ & $.0000^{(\pm.0000)}$\\
    BayesPCN(forget) \cite{yoo2022bayespcn} & $.0188^{(\pm.0002)}$ & $.0176^{(\pm.0001)}$ & $.0019^{(\pm.0000)}$ & $.0008^{(\pm.0000)}$ & $.0465^{(\pm.0001)}$ & $.0235^{(\pm.0001)}$ \\
    MHN & $.1457^{(\pm.0097)}$ & $.0955^{(\pm.0073)}$ & $.4512^{(\pm.0558)}$ & $.4545^{(\pm.0456)}$ & $.5538^{(\pm.0149)}$ & $.6101^{(\pm.0091)}$ \\
    MHN-Manhtn & $.0000^{(\pm.0000)}$ & $.0000^{(\pm.0000)}$ & $.0000^{(\pm.0000)}$ & $.0000^{(\pm.0000)}$ & $.0020^{(\pm.0028)}$ & $.0016^{(\pm.0012)}$ \\
    MHN-GradInf \cite{yoo2022bayespcn} & $.0000^{(\pm.0000)}$ & $.0000^{(\pm.0000)}$ & $.0000^{(\pm.0000)}$ & $.0000^{(\pm.0000)}$ & $.0000^{(\pm.0000)}$ & $.0001^{(\pm.0000)}$ \\
    \textbf{SQHN L1} & $.0000^{(\pm.0000)}$ & $.0000^{(\pm.0000)}$ & $.0000^{(\pm.0000)}$ & $.0000^{(\pm.0000)}$ & $.0000^{(\pm.0000)}$ & $.0000^{(\pm.0000)}$\\
    \textbf{SQHN L2} & $.0002^{(\pm.0001)}$ & $.0000^{(\pm.0000)}$ & $.0000^{(\pm.0000)}$ & $.0002^{(\pm.0002)}$ & $.0001^{(\pm.0000)}$ & $.0008^{(\pm.0002)}$\\
    \textbf{SQHN L3} & $.1076^{(\pm.0022)}$ & $.0000^{(\pm.0000)}$ & $.0000^{(\pm.0000)}$ & $.0000^{(\pm.0000)}$ & $.0000^{(\pm.0000)}$ & $.0000^{(\pm.0000)}$\\
    \bottomrule
    \toprule
    \multicolumn{7}{c}{Recall MSE - High Corruption}\\
    \toprule
    & \multicolumn{2}{|c|}{White Noise} & \multicolumn{2}{|c|}{Pixel Dropout} & \multicolumn{2}{|c|}{Mask} \\
    \toprule
    & CIFAR-10 & TinyImgNet & CIFAR-10 & TinyImgNet & CIFAR-10 & TinyImgNet \\
    \toprule
    BayesPCN \cite{yoo2022bayespcn} & $.0755^{(\pm.0002)}$ & $.0242^{(\pm.0002)}$ & $.0000^{(\pm.0000)}$ & $.0000^{(\pm.0000)}$ & $.0006^{(\pm.0000)}$ & $.0001^{(\pm.0000)}$\\
    MHN-Manhtn & $.0000^{(\pm.0000)}$ & $.0000^{(\pm.0000)}$ & $.1840^{(\pm.0368)}$ & $.1248^{(\pm.0884)}$ & $.7644^{(\pm.2844)}$ & $.7588^{(\pm.2188)}$ \\
    MHN-GradInf \cite{yoo2022bayespcn} & $.0052^{(\pm.0000)}$ & $.0000^{(\pm.0000)}$ & $.3840^{(\pm.0010)}$ & $.5630^{(\pm.0036)}$ & $.3957^{(\pm.0000)}$ & $.6378^{(\pm.0000)}$\\
    \textbf{SQHN L1} & $.0000^{(\pm.0000)}$ & $.0000^{(\pm.0000)}$ & $.0000^{(\pm.0000)}$ & $.0000^{(\pm.0000)}$ & $.0000^{(\pm.0000)}$ & $.0000^{(\pm.0000)}$\\
    \textbf{SQHN L2} & $.0904^{(\pm.0078)}$ & $.0002^{(\pm.0004)}$ & $.0000^{(\pm.0000)}$ & $.0000^{(\pm.0000)}$ & $.0000^{(\pm.0000)}$ & $.0008^{(\pm.0006)}$\\
    \textbf{SQHN L3} & $.3324^{(\pm.0056)}$ & $.0240^{(\pm.0032)}$ & $.0000^{(\pm.0000)}$ & $.0000^{(\pm.0000)}$ & $.0000^{(\pm.0000)}$ & $.0000^{(\pm.0000)}$\\
    \bottomrule
\end{tabular}
\end{adjustbox}
\caption{\textbf{Top} Recall MSE of 1024 images from CIFAR-10 and Tiny ImageNet datasets. At test time, images are either corrupted with white noise (variance .2), $25\%$ of the pixels dropped, or the right $25\%$ of the pixels masked. \textbf{Bottom} Recall MSE of 128 images from CIFAR-10 and Tiny ImageNet datasets. At test time, images are either corrupted with white noise (variance .8), $75\%$ of the pixels dropped, or the right $75\%$ of the pixels masked.}
\vspace{-10pt}
\label{tab:recMSE}
\end{table}

\begin{figure}[t]
\includegraphics[width=.99\textwidth]{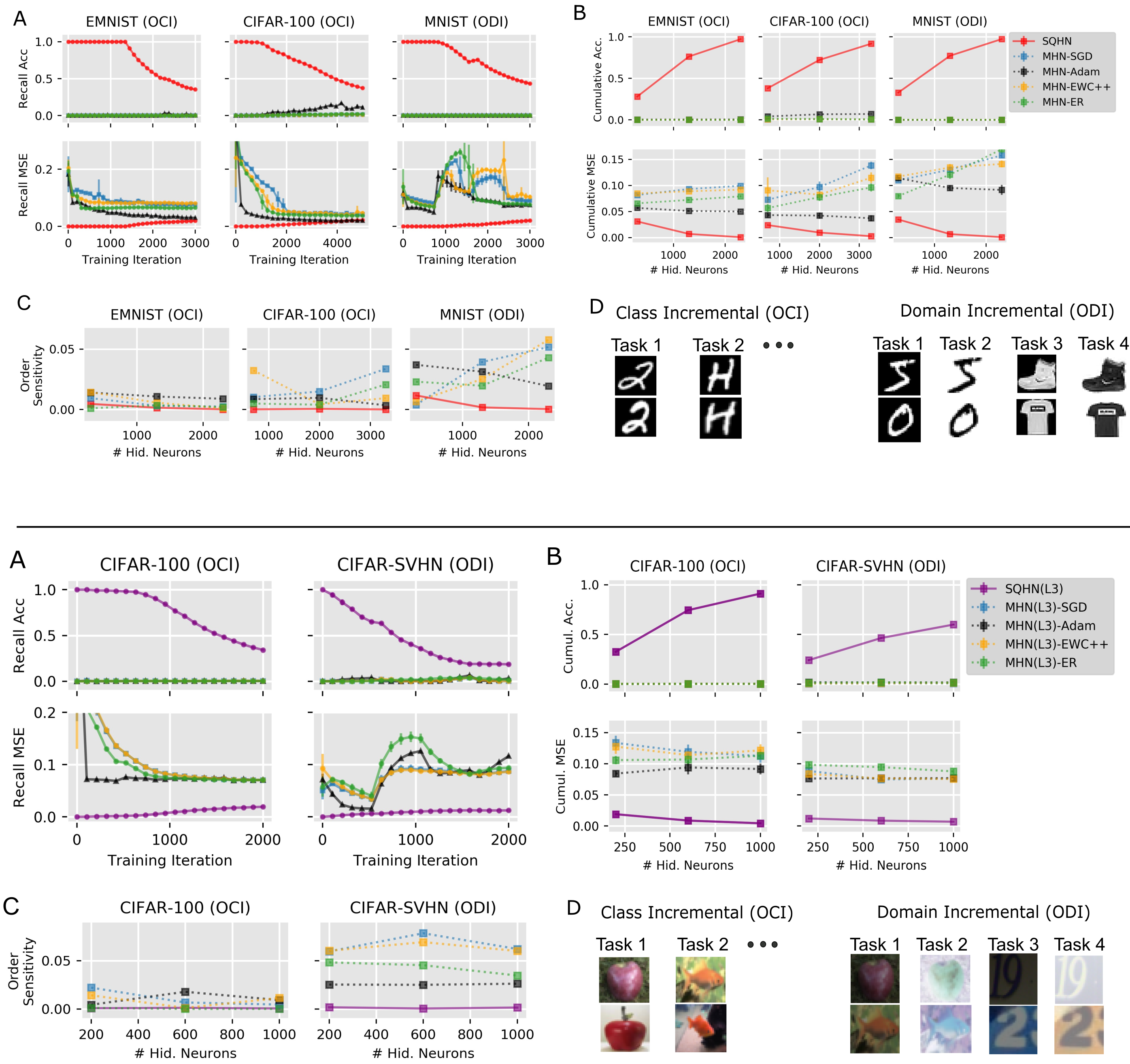}
\caption{Online Continual Auto-Association. \textbf{Top} One hidden layer models with small (300), medium (1300), large (2300) node sizes. \textbf{Bottom} Three hidden layer models  with small (200), medium (600), large (1000) node sizes. \textbf{A} Recall accuracy and recall MSE during training on models with medium size hidden nodes. \textbf{B} Cumulative recall MSE and recall accuracy for each model size. \textbf{C} Order sensitivity for each model size. \textbf{D} Online class incremental (OCI) versus online domain incremental (ODI) settings.}
\centering
\label{fig:onlineCont}
\end{figure}

\subsection{Experiments}

We tested SQHN on several tasks and compare its performance to SoTA and baselines. Specifically, we tested performance on: 1) Auto-Association. 2) Hetero-association. 3) Online Continual Auto-Association. 4) Noisy Encoding. 5) Episodic Memory.

\subsubsection{Auto-Association and Hetero-Association Comparison}
We first compared SQHN models to SoTA associative memory models on auto-associative and hetero-associative recall tasks. For both tasks, unaltered data points from a set $X_{train}$ are presented to the model during training. In auto-association, during testing the model is given corrupted versions, $\tilde{X}_{train}$, of training data, and the model is tasked with reconstructing the original data points. Here, corruption is added to the images with white noise. For hetero-association tasks, during testing portions of the input data are treated as missing. Following, the recent work of \cite{yoo2022bayespcn}, we remove a certain number of pixels from the input data randomly (pixel dropout) or we remove a certain number of the right most pixels (mask).

Three SQHNs are tested: an SQHN with one hidden layer (SQHN L1), two hidden layers (SQHN L2), and three hidden layers (SQHN L3). We compare to two types of SoTA models: predictive coding networks (PCN) and modern Hopfield networks (MHNs). PCNs are neural network models \cite{rao1999predictive}, that implement a kind of probabilistic generative model with continuous latent variables \cite{friston2009predictive}. Three types are compared: offline trained PCN (GPCN) \cite{salvatori2021associative} and two online trained versions (BayesPCN, BayesPCN with forgetting) \cite{yoo2022bayespcn}. Continuous MHNs \cite{ramsauer2020hopfield} have similarities to auto-encoders with a single hidden layer where a softmax activation is used. We compare to the original model of \cite{ramsauer2020hopfield} (MHN), a version of this model by Millidge et al. \cite{millidge2022universal}, which showed better performance by using a Manhattan distance measurement in its recall operation (MHN-Manhtn) (see methods), and the MHN of \cite{yoo2022bayespcn} that used a novel gradient-based inference procedure (MHN-GradInf).  

Results are in table \ref{tab:recMSE}. PCN models struggled on noisy, auto-association task. The MHN that uses Manhattan distance performed very well, and the MHN-GradInf and the one level SQHN model performed perfectly on all auto-association tests. The multi-level SQHNs were more sensitive to white noise on smaller CIFAR-10 images, but still performed well with moderate corruption, and performed very well when they have larger lower layer receptive field sizes, which they use on the larger Tiny Imagenet images.

The GPCN and BayesPCN achieved very low recall MSE on most masking tasks. MHN-grad and the MHN with Manhattan distance performed well on moderate masking, but failed completely on the high masking scenario. \textit{The one and three level SQHN models performed perfectly on all masking tasks, while the two level performed nearly perfectly. SQHN models were the only models to match SoTA performance across both auto-associative and hetero-associative tasks.}

\subsubsection{Online, Continual Auto-Association} 
Next, we test SQHNs on online, continual auto-association. The application for online learning algorithms are typically embedded learning systems (e.g., robots, sensing devices, etc.), where computationally and memory efficient algorithms are preferred. PCN networks are highly computationally expensive, requiring hundreds of neuron updates per training iteration (see \cite{salvatori2021associative, yoo2022bayespcn}), making them impractical as an online auto-associative memory system. Thus, we compare SQHNs to the computationally efficient MHN model. However, we cannot perform the batch update that is typically used in auto-associative memory tests of MHNs. Instead, following previous work, (e.g., \cite{krotov2016dense, ramsauer2020hopfield}), we train MHNs with BP to reduce reconstruction error. As baseline comparisons, we train MHNs with BP/SGD and BP with an Adam optimizer. We also train with several compute-efficient algorithms common to continual learning: online elastic weight consolidation (EWC++) \cite{chaudhry2018riemannian}, which is a kind of regularized SGD, and episodic recall (ER) \cite{chaudhry2019tiny}, which uses a small buffer to store a mini-batch of previously observed data points and SGD to update weights with the mini-batch.

We test on two kinds of online-continual auto-associative tasks: online class incremental (OCI) and online domain incremental (ODI). In both, data from each task is presented incrementally, one task at a time. In the OCI setting, each task consist images from the same class. In the ODI setting, there are four data sets, each are composed of visually distinct images (e.g., dataset 1 has bright images, dataset 2 has dim images, etc.). During training, models perform a single pass over each dataset, observing only a single data point each iteration, before switching to the next dataset. At testing, a noisy version of previously observed data are presented.

Performance is measured using recall MSE ($\mathcal{L}_{MSE}$) and following previous works \cite{salvatori2021associative, yoo2022bayespcn} recall accuracy:
\begin{equation}
\label{eqn:Lrecall}
\mathcal{A}^T = \frac{1}{T} \sum^T_{t=0} \mathbf{1}(\frac{1}{d} \Vert x^{t} - x^{t,new}\Vert^2 < \gamma),
\end{equation}
where the indicator function $\mathbf{1}$ is one if the recall MSE is below threshold, $\gamma$, and zero otherwise. We also use a 'cumulative' (i.e., average) performance measure, which is common in online learning scenarios:
\begin{equation}
\label{eqn:Lmse}
\mathcal{C}_{MSE} = \frac{1}{T} \sum^T_{t=0} \mathcal{L}_{MSE}^t, \text{ }\text{ }\text{ }\text{ }\text{ }\text{ }\text{ }  \mathcal{C}_{Acc} = \frac{1}{T} \sum^T_{t=0} \mathcal{A}^t,
\end{equation}
where $\mathcal{L}_{MSE}^t$ and $\mathcal{A}^t$ is the recall MSE and recall accuracy, respectively, at iteration $t$ given the model parameters and input data at iteration $t$. Cumulative measures are sensitive not just the final performance, but also to sample efficiency, i.e., how quickly the models improves performance. Finally, we use a novel measure of sensitivity to data ordering ($S_{MSE}$):
\begin{equation}
\begin{split}
S_{MSE} &= \vert \mathcal{C}_{MSE}^{OnCont} - \mathcal{C}_{MSE}^{On} \vert\\
\end{split}
\end{equation}
This sensitivity measure is 0 when the model achieves the same cumulative MSE in the online (On) and online-continual (OnCont) settings, and increases as the performance differs. Models insensitive to ordering are highly useful in realistic scenarios where the way data is presented to the model is difficult to control and predict.

Results are shown in Figure \ref{fig:onlineCont}. \textit{SQHNs were highly insensitive to the ordering of the data in online-continual scenarios (Figure \ref{fig:onlineCont} C, top and bottom). They perform one-shot memorization until their capacity is reached, then their performance decays slowly (Figure \ref{fig:onlineCont} A, top and bottom), yielding very good cumulative performance (figure \ref{fig:onlineCont} B top and bottom). The BP-based MHNs learned too slowly to recall any data points and were much more sensitive to ordering, especially in the ODI setting. This lead to poor cumulative scores.} These results provide evidence the SQHN is a highly effective online-continual learner, especially in scenarios where fast learning is essential, and may have general benefits over SGD/BP based approaches.
\begin{figure}[t]
\includegraphics[width=.95\textwidth]{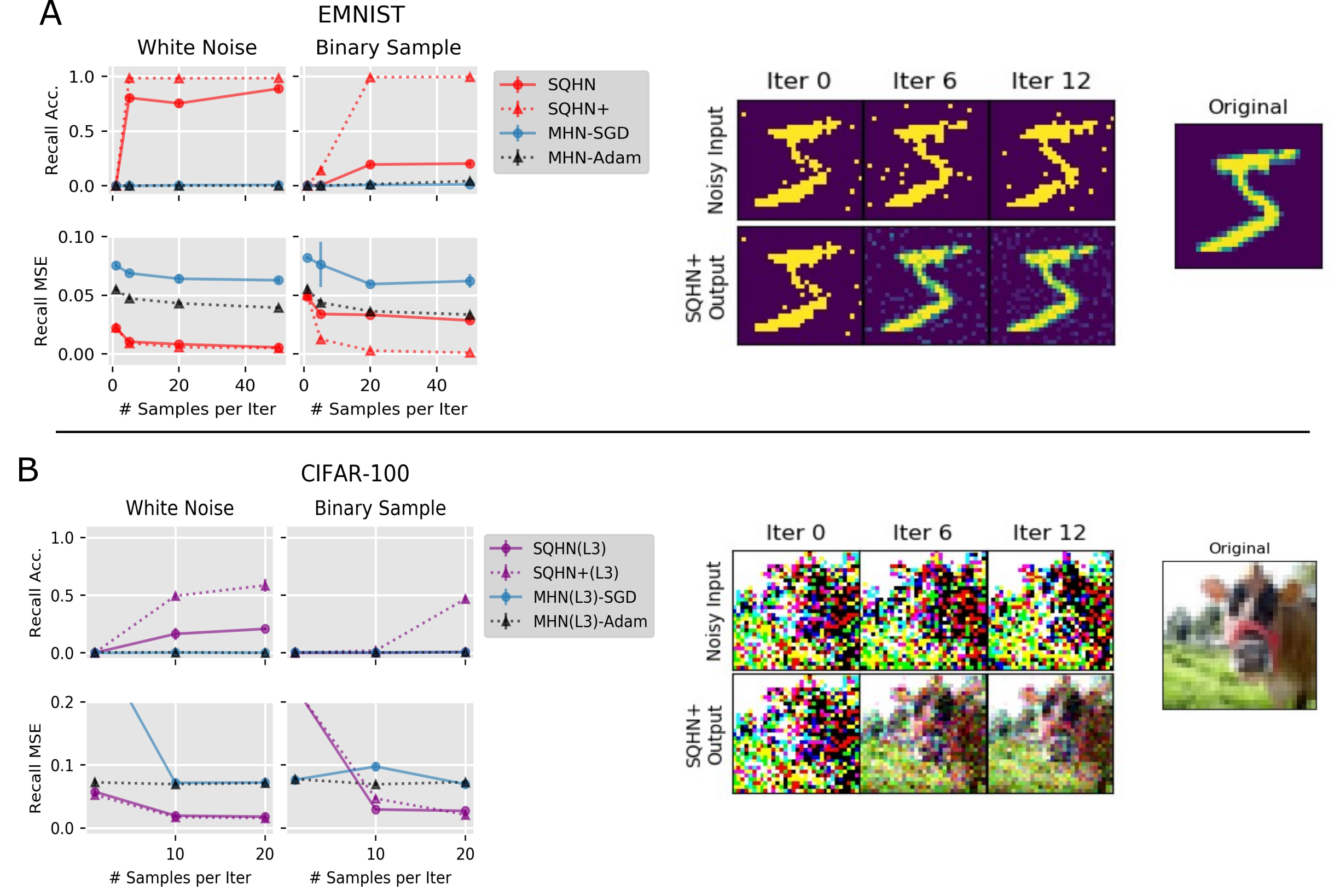}
\caption{Noisy encoding task. \textbf{A.} Recall accuracy and recall MSE for one hidden layer models under white and noise and binary sample conditions. Example of the reconstruction during test time for SQHN+ model, in the case of 1, 6, and 12 samples. \textbf{B.} Recall accuracy and recall MSE for three hidden layer models under white and noise and binary sample conditions with reconstruction example for SQHN on the right.}
\centering
\label{fig:nsEncode}
\end{figure}
\begin{figure}[t]
\includegraphics[width=.89\textwidth]{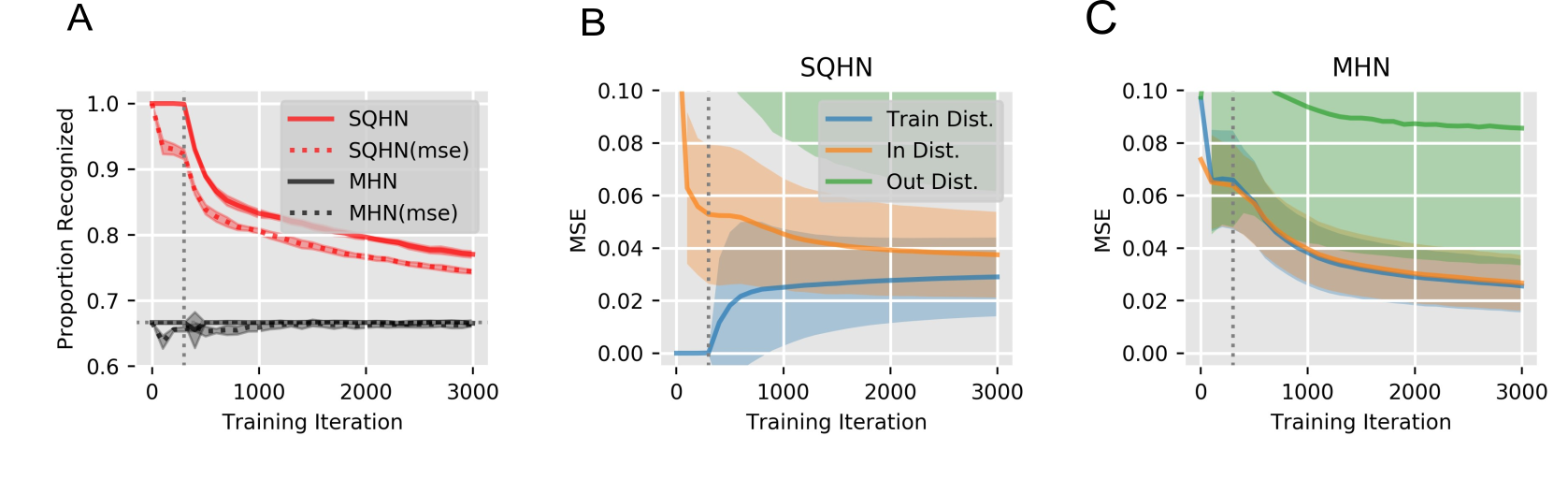}
\centering
\caption{Recognition task. \textbf{A} The recognition accuracy for SQHN and MHN are shown in networks with 300 neurons at the hidden layer (300th iteration marked by vertical dotted line). \textbf{B} The MSE for the SQHN model on the training MNIST data (Train Dist.), hold-out MNIST data (In Dist.), and the F-MNIST data (Out Dist.) \textbf{C} The MSEs for the MHN model.}
\centering
\label{fig:recognition}
\end{figure}

\subsubsection{Noisy Encoding} 
We tested SQHN models on a noisy encoding task, which is novel with respect to recent machine learning work on neural network based associative memory. In this task, each training iteration some number of noisy samples of an image are generated and presented to the network one at a time. Images are Gaussian (white noise added) or binary samples. The models must encode the images, and reconstruct them at test time, when they are presented with the original non-corrupted versions. At no point during training is the model shown the non-corrupted version of the data. The purpose of this task is to specifically test the model's ability to demonstrate unsupervised learning in a noisy setting.

We compared the one and three hidden layer SQHN and MHN models on EMNIST and CIFAR-100, respectively. The one layer models had 300 hidden layer neurons, while the three hidden layer architectures had 150 neurons at each node. Networks were presented with 300 and 150 noisy images, respectively, so an inability to recall the images was more dependent on an inability to remove noise during learning rather than on capacity constraints.  

In addition to testing the SQHN and MHN models, we tested an SQHN model with a slight alteration (SQHN+), where after the hidden states are computed for the first image sample, the hidden states are held fixed for the remaining duration of the training iteration. This ensures the latent code did not change during the encoding of the same image.

\textit{SQHN models, especially SQHN+, were highly effective at removing noise during learning} (figure \ref{fig:nsEncode}), \textit{and both SQHN models significantly outperformed MHNs trained with BP and BP-Adam. This provides evidence the SQHN can be an effective learner under noise and outcompete similar BP-based models.}

\subsubsection{Episodic Memory}

Next, we tested models on a novel episodic memory task based on human and animal memory experiments, which we call episodic recognition. Recognition is hallmark ability of animal memory but has not, as far as we can find, been developed into a formalized machine learning test. In our task, models are presented with a sequence of training data points, $X_{train}$, in an online fashion. Afterward, the model is tested on a binary classification task, where the model must classify the data point presented as old (in the training set) or new (not from the training set). The testing dataset is an equal portion of observed training images, $X_{train}$, new unobserved images from a related (in distribution) data set, $X_{new-in}$, and new images from an unrelated (out of distribution) data set, $X_{new-out}$. A high performing memory system will achieve significantly above chance accuracy, which will decay gracefully as the size of the data set increases.

Since there are no previous baselines to compare SQHN against for the episodic recognition task that we know of, we ran a simple comparison between an SQHN and MHN with single hidden layers (figure \ref{fig:recognition}). Since MHN has not been used for an episodic recognition task, we created two methods for performing recognition in an MHN with one hidden layer. The first method uses the activities at the hidden layer as a measure of similarity to stored data points. The second method keeps a moving average of the recall MSE during training. If the hidden layer activity is above or if the MSE is below a threshold the model judges the data point is new. Grid search is used to set the threshold.

\textit{SQHN models performed perfectly until capacity is reached (vertical dotted line). Performance then decayed gradually, whereas MHN models were unable to do better than chance} (figure \ref{fig:recognition} A). The performance differences seem due to the fact that SQHN models 'overfit' the training data (figure \ref{fig:recognition} B), early in training allowing it to recognize $X_{new-in}$ as less probable than $X_{old-in}$. The MHN model on the other hand performed identically with old and new in-distribution data. This allowed MHN to generalize better earlier, but it prevented the MHN from being able to distinguish previously observed data points from similar data points that were not present in the training distribution.

\begin{figure}[t]
\includegraphics[width=.95\textwidth]{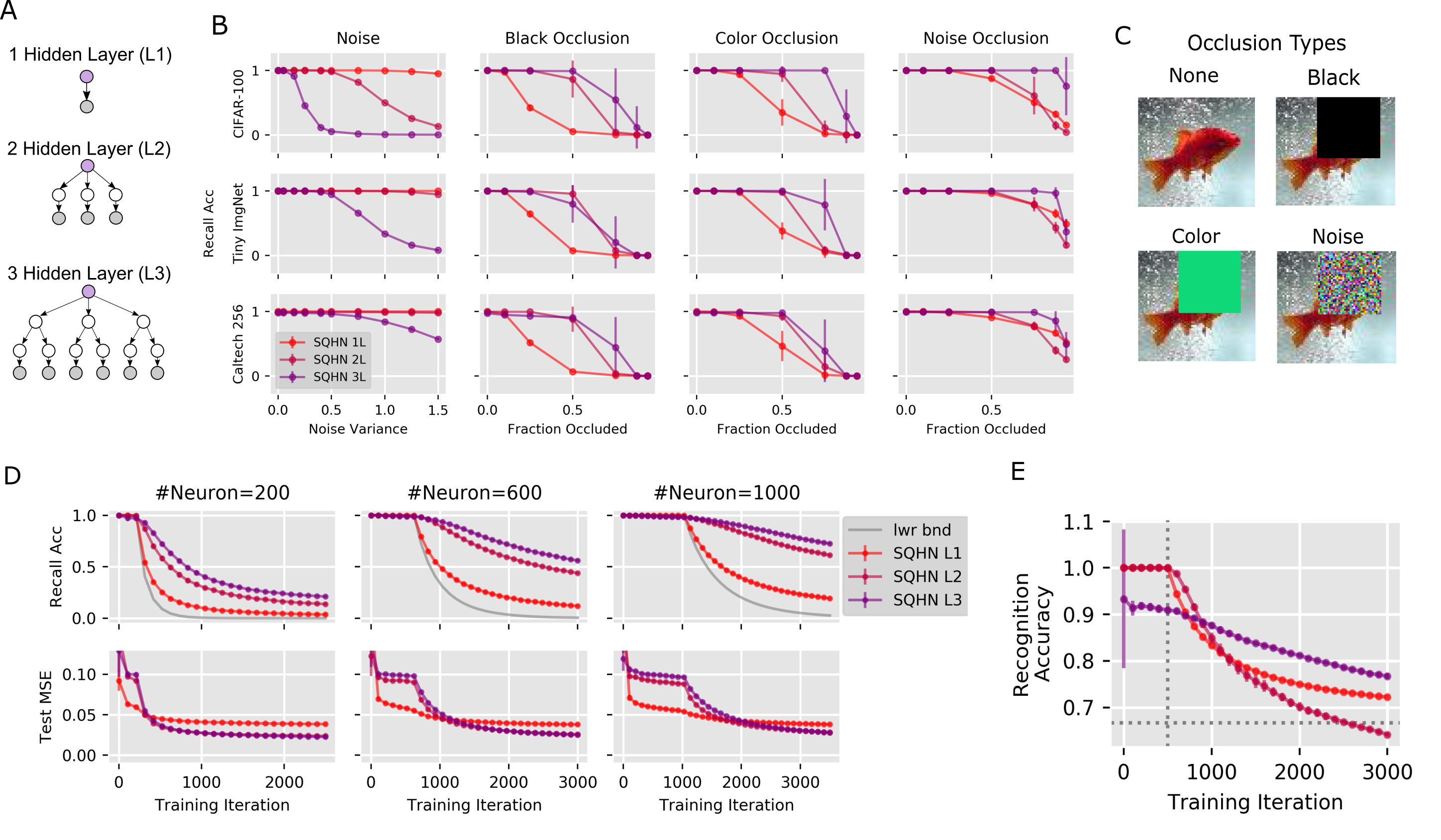}
\centering
\caption{Comparison of SQHN architectures with different hidden layers. \textbf{A.} Depiction of various SQHN architectures.  \textbf{B.} Recall accuracy across three data sets (CIFAR-100, Tiny Imagenet, Caltech 256) under the white noise and several occlusion scenarios. \textbf{C.} Visualization of the black, color, and noise occlusions. \textbf{D.} Recall accuracy during online training without noise (top row) and MSE on a test a test set (bottom row). The lower bound on recall accuracy posited by theorem \ref{thrm:forget} (supp. \ref{supp:theorResults}) marked by gray line. Models tested with different maximum number of neurons per node (200, 600, 1000). \textbf{E.} Recognition accuracy for three SQHN models, where 500 neurons are allocated to each node (vertical dotted line marks when all 500 neurons are grown). CIFAR-100 data used for train and in-distribution set, while a flipped pixel version of the Street View House Numbers (SVHN) dataset used for out of distribution. Best guessing strategy yields $66\%$ accuracy (horizontal dotted line).}
\label{fig:archCompare}
\end{figure}

\subsection{Further Comparison of SQHN Architectures}
Finally, we did a thorough comparison of SQHN models that have one, two, or three hidden layers (figure \ref{fig:archCompare} A), to better establish what the advantages and disadvantages are of adding more hidden layers to the SQHN. 

First, we tested how sensitive SQHN models are to corruption in several auto-association tasks (figure \ref{fig:archCompare} B). During training each model memorizes 1000 images from the CIFAR-100, TinyImageNet, or Caltech256 datasets. During testing images are either corrupted with white noise or, what we call, an occlusion. In the occlusion scenario, pixels in a rectangular region of random shape and position are set equal to either 0 (black occlusion), a random color (color occlusion), or white noise (noise occlusion). (Note this is distinct from masking, since models treat pixels as corrupted rather than missing.) Results are in figure \ref{fig:archCompare} B. All SQHN models performed recall perfectly when corruption was small. Adding more hidden layers tended to improve performance on occlusion tasks, likely because trees represent data as a part-whole hierarchies and therefore can better ignore corrupted parts of the input during inference. Architectures whose bottom hidden layer nodes had larger receptive fields performed better on the noise task. The one layer SQHN had the largest receptive field so it performed the best. For the largest images (CalTech256), however, receptive field sizes at the bottom layers were large for all models (minimum 8x8), and all models performed similarly on noisy recall.

Next, we observed auto-associative recall performance during online (i.i.d.) learning for SQHN models with different numbers of layers and node sizes (figure \ref{fig:archCompare} D). All SQHN models performed one-shot memorization until minimum capacity is reached (which is the number of neurons at hidden nodes, see theorem 1 supp. \ref{supp:theorResults}). Deeper SQHN's recall accuracy decayed at a much slower rate. We suspected this was due to their ability to reuse primitive feature representations at lower nodes to generalize better across multiple data points. We tested these models on test/hold-out data and indeed found the tree architectures generalized better to new data than the one hidden layer model.

Finally, we compared SQHN models on the episodic recognition task. (figure \ref{fig:archCompare} E). Models had 500 neurons at each node. The $X_{train}$ data were CIFAR-10 train images, $X_{new-in}$ were CIFAR-10 images from a hold/out test set, and $X_{new-out}$ were images from the the SVHN dataset with flipped pixels. Figure \ref{fig:archCompare} (E) shows all SQHN models were able to perform well above chance accuracy ($66\%$, horizontal dotted line) even when far more training data was presented than the number of neurons at the memory node. 

In sum, \textit{adding more levels to tree structured SQHNs significantly improves auto-associative recall with occlusion, slows decay of recall accuracy after the network hits capacity, and improves generalization without significant loss in recognition ability. While adding levels can make the SQHN more sensitive to noise, this seems limited to small images.}

\section{Discussion}
Artificial neural networks were originally designed to mimic the way biological neural circuits process information \cite{mcculloch1943logical}. The way biological circuits \textit{learn} to process information, however, is an open question. In particular, it is unknown how the brain uses local learning rules effectively in noisy, online-continual settings. In this paper, we proposed a general local learning approach to the basic task of storing and retrieving patterns in noisy, online-continual scenarios. We proposed using a sparse, quantized neural code to deal with noisy and partial inputs and to prevent catastrophic forgetting, and implementing this strategy via a discrete graphical model that performed MAP learning, an algorithm that uses local learning rules. We implemented this approach in the novel SQHN model.

Our results support the effectiveness of our approach and the SQHN. First, we found the quantized neural code of the SQHN was advantageous in auto-associative recall over similar models that use a continuous latent codes. In particular, PC models, like the SQHN, implement directed graphical models and learn via MAP learning \cite{millidge2022theoretical}. However, because PC models use a continuous latent code, they were much more sensitive to noise than SQHN models. This lends credence to the idea that using a quantized neural code helps significantly with auto-association. MHN models performed similarly to SQHNs on the noise auto-association task. This is likely because under the hyper-parameter settings that yielded the best performance, these MHNs essentially implemented a discrete latent code (supplementary \ref{supp:SQHNvHop}). However, the operation MHNs use to map inputs to latent codes, is not as effective as that of SQHNs in the high masking settings.

Second, we found SQHNs significantly outperformed similar MHN architectures trained with BP-based algorithms on online-continual learning tasks. The SQHN trained faster, demonstrated one-shot memorization, showed only a small, stable forgetting rate, and was largely insensitive to ordering. It achieved all of this while using little extra memory, no episodic memory buffer, and little compute. Part of these performance advantages may be attributable to the parameter isolation approach generally. However, the SQHN also uses an effective learning rate schedule to prevent forgetting (fig.\ref{fig:onlineContAbl}), and we also find mathematically that using MAP inference to set the one-hot values at hidden nodes is a principled way of deciding which parameters to update. In particular, it yields a set of activities and updates that require only a small change to existing parameters (supp. \ref{supp:theorResults}). This suggests the exciting prospect that MAP learning in discrete models, like SQHNs, provides a justified and bio-plausible way to perform parameter isolation, which unlike previous parameter isolation methods (e.g.,  \cite{lee2020neural, yoon2017lifelong, mallya2018packnet, mallya2018piggyback}), does not require the computation of non-local global loss gradients to decide how to isolate parameters.

Third, the sparse quantized code reduced the negative effects of noise during the noisy encoding process by yielding a hidden latent code that was largely stable across noisy samples. If we fixed the latent code to be perfectly stable, as we did with SQHN+, then performance improves even more. The sparse updates also helped prevent noise from interfering with previously recorded memories.

Finally, the SQHN proved to be highly effective in the episodic recognition task. The memory node of the SQHN stored an explicit, itemized record of previously observed feature representations of input data (see supp. \ref{supp:recog}). During recall the memory node performs a nearest neighbor operation by finding the item with the highest energy. This operation turned out to yield a straightforward method for detecting new versus old data points, is similar to classic cognitive models of episodic recognition (e.g., \cite{shiffrin1997model}), and it performed well even when the memory node was pushed past its capacity. Models like the MHN trained with BP, on the other hand, did not naturally learn an explicit record of previously observed feature prototypes and struggled to distinguish new from similar, old data points.

Importantly, SQHN's energy function yields a straightforward way to implement a directed graphical model with largely standard artificial neural network operations. In particular, since the energy was a sum of probabilities, rather than a product, neurons in the SQHN architecture summed inputs from children and parent nodes rather than multiplied (as belief networks do \cite{bishop2006pattern}), yielding standard neural network operations which are easily scaled and more compatible with hardware built specifically to handle vector matrix multiplies. SQHN does so without the need to move probabilities to the log domain, which can sometimes yield unusual properties like the need to represent very large negative numbers (i.e., overflow issues). The energy can also be justified as an approximation to the joint distribution when uncertainty (i.e., a prior distribution) is placed over parameters (supp. \ref{supp:energy}). Taken together these results suggest that our general approach, and the SQHN implementation of it, could provide a highly promising basis for building neuromorphic online-continual learners.

It is also interesting to note the striking similarities between SQHNs and models of memory from neuroscience. In particular, in addition to utilizing sparse codes and local learning rules like the brain, the memory node of the SQHNs has similarities to hippocampus (HP). HP is a region of the brain closely tied to episodic memory. We showed in our experiments the memory node, like the HP, is highly effective for episodic recognition. Further, the HP, like the memory node in SQHNs, is often proposed to be at the top of the cortical hierarchy \cite{mcnaughton2010cortical}, and some theories propose a central function of the HP is to retrieve and return previously stored patterns given partial, noisy inputs from cortex \cite{teyler1986indexing, teyler2007indexing}. This is precisely what the memory node of SQHNs does (see figure \ref{fig:archDiagram}). Maybe most interesting, we find that on average neurons grow more rapidly and for longer periods in the memory node than the rest of the network and synapses in the memory node tend to have larger step sizes (more flexibility) on average (fig. \ref{fig:emerge}). This is highly consistent with observations that HP grows neurons into adulthood while the rest of the cortex stops in early adulthood \cite{ming2011adult}, and the observation that synapses are highly flexible in HP compared to the rest of the cortex \cite{kumaran2016learning}. Importantly, the SQHN was not engineered to have these properties. Rather these properties emerged from the SQHN learning algorithm as it solved the online-continual learning problem. 

Future work will need to assess whether and how the SQHN provides possible new ideas or insights into these topics in neuroscience. Further, on the machine learning side, future work will need to assess the SQHN on tasks that require generalization. Although we found SQHNs with simple tree architectures to be highly efficient at storing and retrieving training data, and preventing forgetting, we believe somewhat more complicated SQHN architectures will be needed for high performance on tasks like classification or self-supervised learning where generalizing to new data points is needed. Nonetheless, the results here suggest SQHNs could provide a promising approach for learning these tasks in the online-continual setting.

\bibliographystyle{plain}
\bibliography{sample}

\appendix
\newpage
\section{Methods}

\subsection{Code}
Code will be publicly released upon publication of the paper. Code was written in Python 3.7.6 using Pytorch version 1.10.0 to implement models and perform differentiation for models that use backpropagation. All simulations were run on a small GPU type NVIDIA GeForce RTX 270 with MAXQ design.

\subsection{Datasets and Hyperparameters}
Image values ranged between 0 and 1, unless otherwise noted. Images are converted to pytorch tensors, but no normalization or other alterations where made unless otherwise specified. MNIST, Fashion-MNIST, and E-MNIST are image data sets with images sized 1x28x28. SVHN, CIFAR-10, and CIFAR-100 are natural image data sets with images sized 3x32x32. Tiny ImageNet is a natural image data set with images sized 3x64x64. Finally, CalTech256 is a natural image data set with images of various sizes. We cropped all CalTech256 images to 3x128x128. For all models, we use a grid searches to find hyper-parameters.

Unless otherwise specified, we use a recall threshold, $\gamma$, of .01. Prior works (e.g., \cite{salvatori2021associative, yoo2022bayespcn}) use smaller thresholds of .005 or .001. We use a slightly larger threshold here, since we find one of our main comparison models, MHN, is unable to recall any images while training with BP, and we wanted to show this was not simply a result of an arbitrarily small recall threshold. 

\subsection{SQHN Implementation Details}
All SQHN architectures are set up to have a tree structure. One can think of the structure as being similar to locally connected networks or convolutional networks without weight sharing. Thus, we can talk about SQHNs as having a certain number of channels at each hidden layer (equivalent to the number of neurons at each node) and the receptive field size of each channel/node. To simplify the architectures, we design SQHNs so that nodes have non-overlapping receptive fields, which means each node has one parent. More complex versions of SQHNs can be used for more complex vision tasks, but we found these simpler architectures performed very well and eased scaling in associative memory tasks. Here we explain how inference is implemented. Inference involves a single feed-forward and feedback sweep through the network. The goal of inference is to maximize energy, i.e. the sum of conditional probabilities of input and hidden node values (equation \ref{eq:energy}).

Let $h_{l,j}$ be the internal state value of the $jth$ neuron in the $l$th node. Let the $l$th node be in the first hidden layer, which has one child node. Let $M_l$ be the matrix from node $l$ to its child node, which is visible and clamped in image patch $x_{c_l}$. The matrix $M_l$ can be decomposed in to columns or 'memory vectors': $M_l = [m_{l,0}, m_{l,1},...,m_{l,J}]$. During the feed forward pass $h_{l,j}$ receives the signal
\begin{equation}\label{eq:inpLayerBot}
h_{l,j} = \frac{.5 (m_{l,j}^{\top} - .5) (x_{c_l} - .5)}{\Vert (m_{l,j}^{\top} - .5) \Vert \Vert (x_{c_l} - .5) \Vert} + .5.
\end{equation}
This operation is a kind of shifted cosine similarity operation between in input and each memory vector $m_{l,j}$. We show in supplementals this operation can be interpreted as a weighted, normalized average of the probability of each pixel value, under the assumption each pixels value is a binary variable (supp. \ref{supp:infer}). In matrix form this operation only involves a vector matrix multiply, like a neural network, with an extra shift operation and normalization operation:
\begin{equation}\label{eq:inpLayerBotMat}
h_{l} = \frac{1}{2 Z} (M_{l}^{\top} - .5) (x_{c_l} - .5) + .5,
\end{equation}
where $\frac{1}{Z}$ is the neuron-wise normalization. In practice, one could also store a separate feedback matrix with the shifted, transposed, and normalized version of matrix $M_l$.

Hidden nodes at the 2nd layer and higher have children nodes that represent discrete/categorical variables. Computing the sum of probabilities of each child therefore requires a different operation. Let $h_{c_n}$ be the vector of internal neuron values of the nth child node of node $l$, and let $h^{max}_{c_n} = max(h_{c_n})$, where $max$ is an activation function that return a vector of all zeros except for the maximum value, e.g., $max([.1, .8, .4]) = [0, .8, 0]$. Let $h^{max}_{c}$ be the concatenation of all $N$ children node max values: $h^{max}_{c} = [h^{max,\top}_{c_0}, h^{max,\top}_{c_1}, ...h^{max,\top}_{c_N}]^{\top}$, and let $M_l$ be the concatenation of the matrices lead from node $l$ to its children: $M_l = [M_{l,0}^{\top}, M_{l,1}^{\top},... M_{l,N}^{\top}]^{\top}$. The update for a hidden node is
\begin{equation}\label{eq:hidLayerBot}
h_{l} =  \frac{1}{Z} M_l^{\top} h_c^{max},
\end{equation}
where $Z = N \Vert h_c^{max} \Vert$. We show in the supplemental (section \ref{supp:infer}) this is equivalent to taking the weighted average of the conditional probability of child node assignments, where each conditional probability is weighted by the weighted average of the probability of its children values, which are a weighted probability of their children values, and so on. Using this weighting thus conveys information about the probabilities of all descendents.

After signals propagate to the root/memory node, signals are propagated down the tree, and one-hot values assigned to each node. Let $argmax$ be the activation function that assigns a one-hot to the maximum value, e.g., $argmax([.1, .8, .4]) = [0, 1, 0]$. Let $h^*_l$ be the one-hot assignment for the $l$th node. At the memory node the assignment is simply
\begin{equation}
h^*_L = argmax(h_L).
\end{equation}
Signals are then propagate down according to 
\begin{equation}
h^*_l = argmax(\lambda h_l + (1 - \lambda) M_{pa_l,l} h^*_{pa_l}),
\end{equation}
where $pa_l$ is the parent node of $l$ and $\lambda$ a hyper-parameter. For all recall tasks we set $\lambda = .5$.

\subsection{Associative Recall Comparison} For our initial comparison to SoTA associative memory models we tested our SQHN models on the same associative memory task as \cite{yoo2022bayespcn}. Each SQHN model had the same number of neurons at its nodes as there are images. SQHN L2 had kernal sizes 4x4 and 8x8 for CIFAR and tiny image net respectively at its bottom layer. It had kernal size 4x4 at second layer for both data sets. SQHN L3 had kernal sizes 2x2 and 4x4 for CIFAR and tiny image net respectively at its bottom layer. It had kernal size 4x4 at both its second and third layer for both data sets. In the moderate corruption task, 1024 images are presented. In the high corruption task, 128 are presented. The auto-associative task used white noise corruption with variance .2 and .8 for moderate and high corruption task, respectively. For hetero-associative tasks, $25\%$ and $75\%$ of pixels are masked for moderate and high tasks. Masked pixels are treated as missing in these tasks and are therefore ignored by the bottom hidden layer of SQHN (rather than being treated as 0 values, which affects the normalization term). MSE is only computed in hetero-associative tasks between the output and the original pixels that were missing in the input. 

Yoo et al. \cite{yoo2022bayespcn} tests the generative predictive coding network (GPCN) of \cite{salvatori2021associative}, which trains offline. Yoo et al. develops an version for online training, called BayesPCN, and BayesPCN with forgetting, which prevents learning from slowing too much. 

We also compared to MHN models. Following previous works (e.g., \cite{millidge2022universal, yoo2022bayespcn}, we update MHNs with a simple batch update where all the input vectors are stored in columns of a 'memory matrix' $M$ and we set the temperature, $\beta$ to 10000, essentially treating it as a argmax operator. The original model of \cite{ramsauer2020hopfield} performs retrieval via the operation 
\begin{equation}
x^{(new)} = M softmax(\beta M^{\top} \tilde{x}),
\end{equation}
where $\beta$ is the temperature. Millidge et al. \cite{millidge2022universal} showed this recall process is a kind of nearest neighbor operation, where a weighted average of memory vectors are returned, and memory vectors more similar to the input, according to a dot product measure, are given more weight. Recall is typically best when $\beta$ is large, and the nearest memory vector is returned. Millidge showed that other similarity measures work better than dot products. We show the best performing model of Millidge et al. which uses the Manhattan distance. Finally, we also report the results of the gradient-based MHN (MHN grad) listed in \cite{yoo2022bayespcn} that performs recall via a gradient-based inference procedure instead of the one-shot recall process shown above.  

An important caveat is that the results reported by \cite{yoo2022bayespcn} used images normalized to -1 and 1. Our SQHN model was designed for non-normalized images with values 0 to 1. To make the comparison fair, we multiplied the MSE scores for the SQHN by 4 since normalized images (ranging from -1 to 1) have an error range twice as large as unnormalized images (ranging from 0 to 1), which is then squared in the squared error: let $e = x - x^{new}$ be the unnormalized image and output error. If the image and output were normalized before computing the error we would be $2e$. If this error is then squared we get $4e^2$. Further, standard deviation must be multiplied by two.

\subsection{Online Continual Auto-Associative Memory Tests (Figure \ref{fig:onlineCont})} In the online-continual learning task, images were presented online (one at a time, and single pass through the data) in a non-i.i.d. fashion. In particular, images from the first task are presented online, then images from the second task are presented online, and so on. We tested grouping images by class (online class incremental or OCI) and by visually distinct data sets (online domain incremental or ODI). Each class and domain group had an equal number of images. In the ODI setting, the one layer models trained one four different data sets/domains: MNIST, MNIST with flipped pixels, FMNIST, and FMNIST with flipped pixels. One layer models were tested with 300, 1300, and 2300 on MNIST data sets, and 700, 2000, 3300 on CIFAR-100. The three layer models trained on CIFAR-10 with dark pixels ($x^{dark} = x * .5$), CIFAR-10 with bright flipped pixels ($x^{light} = (-1 * x + 1) * .5 + .5$), SVHN with dark pixels, and SVHN with bright flipped pixels. The three layer SQHN model had kernal sizes 4x4, 4x4, 2x2, at the first, second, and third hidden layers and was tested with node sizes 200, 600, 1000. During test time training images with noise variance .2 presented.

We compare to another auto-associative memory model known as the modern Hopfield network (MHN) \cite{krotov2016dense, ramsauer2020hopfield}. We compare to the MHN of \cite{ramsauer2020hopfield}, in particular, because it is the most similar architecturally to SQHN, and has a very high capacity relative to other MHNs \cite{ramsauer2020hopfield}. The one layer model performs recall as:
\begin{equation}
x_{new} = M softmax(\beta M^T x),
\end{equation}
where $x$ is the input vector and M with weight matrix, and $\beta$ a scalar temperature. This same mechanism can be stacked into tree-like hierarchies as well, similar to the multi-level SQHN. For the multi-level MHN we essential use the same architecture (same kernal sizes and node sizes) but replace the argmax with softmax and remove the normalization.

One issue we ran into was there was no standard way to train this model online for auto-associative tasks like the ones used here. The original model of \cite{ramsauer2020hopfield} was trained with BP, but not used for auto-associative memory. Instead, it was treated like the attention layer of a transformer, where $M$ is generated by a separate set of weights. Thus, we train the model like the original MHNs of \cite{krotov2016dense}, which were used for auto-association, where BP is used to directly optimize $M$. Since we are training to reduce recall MSE we directly optimize MHNs to reduce recall MSE. We test the model trained with plain SGD, SGD with Adam optimization \cite{kingma2014adam}, EWC++ \cite{chaudhry2018riemannian} which is an online version of EWC \cite{kirkpatrick2017overcoming}, and episodic replay with a small memory buffer \cite{chaudhry2019tiny}. EWC++ uses a moving average of the Fisher information matrix, $F$, to modulate parameter updates. Parameter updates for EWC are regularized loss gradients where the objective, $\tilde{L}$ is:
\begin{equation}
\tilde{L}^k = L^k + \lambda \sum_i F_{\theta_i^{k-1}} (\theta_i - \theta_i^{k-1})^2,
\end{equation}
where $\theta_i^{k-1}$ is the ith value of the parameters from task $k-1$ and $\lambda$ weight the regularization. In EWC++, the Fisher is computed as a moving average $F^t = \alpha F^t + (1 - \alpha) F^{t-1}$, where $t$ is the current training iteration, and a moving average of parameters is kept as well. At the end of each task, $F_{\theta_i^{k-1}}$ is set equal to $F^t$ and the moving average of parameters set equal to $\theta_{k-1}$. One issue is that EWC++ needs to know task boundaries, so it can store the parameters and Fisher matrix from previous task. In our task, no such supervision is allowed. To deal with this, we just treat each data point as its own task and compute a moving average of the parameters, same as the fisher information matrix. This should still allow parameters and Fishers to maintain information about previous tasks. The $\lambda$ and $\alpha$ hyper-params are found via grid search. Generally, we find the regularization provides little to no benefit, like because the training runs are short and therefore training speed is as important as preventing forgetting. EWC slowed down training speed too when $\lambda$ was significant, so the regularization provided little help.

Episodic replay with a tiny memory buffer \cite{chaudhry2019tiny} stores a small sample of previously observed data-points, then uses SGD to update parameters each iteration using the whole sample as a mini-batch. Despite using a small buffer this approach has shown to be highly effectively in continual classification. Methods that sample a mini-batch from a larger buffer can work better \cite{mai2022online}. However, we are interested in memory and compute efficient algorithms and are using relatively small data sets. Thus, we use the tiny memory buffer method. There are several algorithms for deciding when to store a data point in the buffer. We use the resevoir sampling method (see \cite{chaudhry2019tiny}), which is simple method that has shown to consistently be better than other methods on classification tasks. Resevoir method is simple: let $n$ be the maximum number of data point that can be stored in the buffer, and $t$ be the total number of data points observed. For the data point, $x^t$, at each iteration $t$, check if the buffer is full. If it is not full, add the data point to the buffer. If the buffer is full, randomly overwrite one data point in the buffer with $x^t$, with probability $\frac{n}{t+1}$.

\subsection{Noisy Encoding} We compared one and three hidden layer models on EMNIST and CIFAR-100, respectively, in the noisy encoding task. The single layer models were tested with node size 300, and were presented with 300 images. Trees had node size 150 and were presented with 150 images. Therefore, any inability to recall images must be due to an inability to remove noise during learning, rather than capacity limitations. Images were either Gaussian samples (white noise added then clamped to range 0 to 1) or binary samples, were drawn from the original image values. For each image, some number of samples were drawn and presented to the models online in a sequence, before moving to the next image. For EMNIST images, we test 1, 5, 20, and 50 samples. For CIFAR-100 we tested 1, 10, 20 samples. At no point during training was the original, non-corrupted image presented to the model. At test time, the original, non-corrupted image was presented and the model had to reconstruct. We also tested an SQHN model with a slight alteration (SQHN+), where after the hidden states are computed for the first image sample the hidden states are held fixed for the remainder of the samples. This ensured that samples from the same image were encoded to the same latent state. Although SQHN was able to do this well on its own, it sometimes mapped samples from the same image to different latent state in the high noise, binary sample scenario. In these cases, SQHN+ performed better. 

\subsection{Episodic Recognition} 
The the episodic comparison between SQHN and MHN we use a one layer model and MNIST data sets. The train set is pulled from the MNIST training data. The in-distribution hold out set is from MNIST test data. The out of distribution set is from F-MNIST. We test two methods for performing recognition in an MHN with one hidden layer. The first method uses the activities at the hidden layer as a measure of familiarity/similarity. If that value is above a threshold, $\rho$, then the model judges it has observed the data point. The second method keeps a moving average, $\mu$, of the recall MSE during training, and uses a scalar multiple of this average, $\rho \mu$, as the decision threshold. All hyper-parameters found with grid search. Learning rates updated to achieve best performance in terms of recall MSE. $\rho$ updated to achieve best recognition accuracy. We test SQHN and MHN networks with 300 neurons. 3000 data points from EMNIST are presented in an online and i.d.d. fashion. At each test point all previously observed train data are presented along wth an equal number of in-distribution and out-of-distribution data. We plot the MSEs of MHN and SQHN for each dataset to show that MHN generalize so well that they have no performance different between training and in-distribution data, making it impossible for the model to perform recognition. 

\subsection{SQHN Architecture Compare (Figures \ref{fig:archCompare} B,D,E)}

\textbf{Auto-association with Varying Amounts of Corruption (Figure \ref{fig:archCompare} B)} All SQHN models had 1000 neurons at each node and are trained to memorize 1000 images. SQHNs are trained so that each image is encoded in a unique set of one hot vectors at hidden nodes, (which can be done in practice by setting $\alpha$ to a very large number). This means that all SQHN models are not over capacity, and any inability to recall images is due only to their inability to handle varying amounts or types of corruption. Importantly, only auto-associative, and not hetero-associative, recall is tested. The masks added to images are treated as corrupted pixels (i.e., their values are taken as input) rather than missing pixels. For noise tasks we add white noise to the images then clamp the images to values between 0 and 1. Noise variances tested are $[0, .05, .15, .25, .4, .5, .75, 1, 1.25, 1.5]$. Fractions masked that are tested are $[0,.1, .25, .5, .75, .875, .9375]$. For masking tasks a rectangle with random length and width (less than or equal to the width and height of image) is sampled and its position within the image is randomly sampled. The black mask sets pixels equal to 0. The color mask randomly selects an RGB value from a uniform distribution, and sets pixels in the mask equal to that value. Noise mask sets mask values equal to a white noise sample (variance 1), clamped to values between 0 and 1. The one level SQHN-L1 model has input layer kernal size equal to the size of the image. SQHN-L1 essentially performs nearest neighbor operations comparing each images to the set of training images, which stored in columns of its weight matrix, using mean-shifted cosine similarity. SQHN-L2 has an input layer kernal size of 8x8, 16x16, and 32x32 for CIFAR-10, Tiny Imagenet and Caltech 256, respectively. The second layer kernal across all data sets is sized 4x4. SQHN-L3 has an input layer kernal size of 4x4, 8x8, and 16x16 for CIFAR-10, Tiny Imagenet and Caltech 256, respectively. The second and third layer kernal sizes across all data sets are sized 4x4. 

\textbf{Online Auto-associative Learning (Figure \ref{fig:archCompare} D)} For this task a one, two, and three hidden layer were model were trained in the online i.i.d. scenario on CIFAR-100. Models were trained with 200, 600, 1000 neurons at hidden nodes. SQHN L1 model had input layer kernal size equal to the dimension of the image (32x32). SQHN L2 had input layer kernal size 4x4 (which worked better than the 8x8 kernal used in the experiment above), and second layer kernal size of 8x8. SQHN-L3 has an input layer kernal size of 4x4, 8x8, and 16x16 for CIFAR-10, Tiny Imagenet and Caltech 256, respectively. The second and third layer kernal sizes across all data sets are sized 4x4. $\alpha$ was set very large so networks memorized first J data points, up until capacity reached. We measured recall accuracy without any corruption, and noticed L2 and L3 architectures performed much better than L1. We suspected this was because the L2 and L3 architectures learn representation of small features that generalized widely across training set. To test this we also measured test MSE on the test set from CIFAR-100.

\textbf{SQHN Episodic Recognition Comparison (Figure \ref{fig:archCompare} E)} For the recognition task, a train set, in-distribution hold-out set, and out of distribution set are needed. We use CIFAR-10 training images for the train set, a hold-out/test set of CIFAR-10 images as the in-distribution set, and CIFAR-100 with flipped pixels as out of distribution set (i.e., each image is multiple by -1 then 1 is added). SQHN L1 is given 500 neurons at its hidden node. SQHN L2 and L3 are given 500 neurons at their memory node. Since only the memory node is used directly for recognition, we pre-train the lower layers of SQHN L2 and L3 for 1000 iterations on CIFAR-10 images. SQHN L2 has kernal sizes 4x4 and 8x8 at hidden layers, and channel size of 40 at the first hidden layer. SQHN L3 has kernal sizes 2x2, 4x4, 4x4 at the first, second, and third hidden layers, respectively and channel sizes of 40 and 200 at the first and second hidden layers. After pre-training during the training phase, images from the train set are presented online in and i.i.d. manner. Only the weights leading into the memory node are updated. During testing, models are presented with a set of images, which are composed of all of the training images observed so far, an equal number of in-distribution images, and an equal number of out-of-distribution images. The best guessing strategy is to guess all data points are new which yields $66\%$ accuracy.

\newpage
\section{Supplementary Material}

\subsection{Related Works}
In this subsection we do a more thorough comparison of SQHN to similar models and methods.

\textbf{Bio-Inspired Deep, Sparse Networks} The model most similar to SQHN is likely the Recursive Cortical networks (RCN) \cite{george2017generative}, which is another neural network/Bayesian network hybrid that uses an algorithm akin to MAP learning in a tree-structured architecture. SQHNs use distinct learning rules than RCNs. RCNs for example do not perform the averaging operation that SQHNs do and as a result RCNs cannot learn under noise (e.g., see supplemental from \cite{george2017generative}). RCNs also have more complex architectures that utilize max-pools and factorize color and shape representations. RCNs also do not explicitly minimize an energy function. Further, as far as we know RCNs have not been tested on natural images, or on associative memory tasks from recent machine learning literature. Hierarchical temporal memory \cite{ahmad2015properties} models (HTMs) are another deep neural network model, which, like SQHN, use sparse coding and have similarities to Bayesian networks with tree-like architectures. However, HTMs unlike SQHNs as far as we know, do not explicitly try to implement a discrete graphical model. They also use a different inference procedure, based on thresholding activations rather than one-hots, and a distinct learning algorithm that was not designed to optimize any energy function. Finally, sparse Hopfield-like networks, like those of \cite{o2017deep, rozell2008sparse}, also have some architectural similarities to SQHN too, but these models do not explicitly encode discrete random variables, they use distinct energy functions, and they have not been applied to memory tasks. 

\textbf{Online-Continual Learning} Following previous reviews \cite{parisi2019continual, khetarpal2022towards, wang2023comprehensive, parisi2020online, gallardo2021self, mai2022online, hayes2022online}, continual and online-continual learning approaches may be split into several types: 1) Regularization based approaches constrain the way parameters are update to reduce catastrophic forgetting (e.g., \cite{kirkpatrick2017overcoming, ritter2018online, zenke2017continual, aljundi2018memory, li2017learning}). 2) Memory-based approaches store or model previously observed data for the purpose of replaying the data during training (e.g., \cite{chaudhry2019tiny, shin2017continual, aljundi2019gradient}). 3) Parameter isolation models avoid forgetting by allocating different parameters for each task, either by gating components or by dynamically adding new sub-networks as needed (e.g., \cite{lee2020neural, yoon2017lifelong, mallya2018packnet, mallya2018piggyback}). Our SQHN model differs from these approaches since it uses local learning rules, while these previous approaches us BP in some form (e.g., all the methods reviewed by \cite{parisi2020online, mai2022online} use BP either in pre-training and/or during online training of the classifier). Our SQHN model can be described as using a kind of parameter isolation (via sparse neuron activity and neuron growth) combined with regularization (via learn rate decay schedules) to avoid forgetting, but as far as we can tell SQHNs particular strategy is novel. For example, whereas all the previous works on regularization and isolation use some form of \textit{global} gradient information to regularize and isolate parameters, the SQHN does so based solely on information \textit{local} to the neuron or synapse.

\textbf{Associative Memory} Like the SQHN, classic Hopfield networks \cite{hopfield1982neural}, modern Hopfield networks (MHN) \cite{krotov2016dense, ramsauer2020hopfield}, and predictive coding (PC) networks \cite{rao1999predictive} are all energy based models that can perform auto-associative tasks. SQHN, however, is the only model to explicitly implement discrete graphical model and the SQHN optimizes a novel energy function. Further, unlike the SQHN none of these models (with maybe the exception of \cite{yoo2022bayespcn} which we compare to SQHN below), were designed specifically for online-continual learning.

In sum, although our SQHN model combines some elements from existing model, it utilizing novel, bio-plausible, backprop-free strategies for encoding memories for associative memory tasks. Further, it should be emphasized unlike previous works we are training networks on a largely novel unsupervised noisy encoding memory tasks and recognition tasks, which as far as we can find, have not been directly explored in the machine learning literature.

\subsection{Derivation of Energy}\label{supp:energy}
Consider a directed tree-structured acyclic graph (DAG) where each hidden node represents a discrete random variable. The joint probability of the nodes values, given the learned parameters, is the product of the conditional probabilities of each node given the values of its parent nodes:
\begin{equation}
p(h^*_0, h^*_1,...h^*_L) = \prod_{l=0}^L p(h^*_l| pa^t_l),
\end{equation}
where $pa_l$ refers to the values of the parents of nodes $l$ and $h^*_l$ is the integer value assigned to node $l$. We derive our novel energy (equation \ref{eq:energy}) function by performing inference using not just the learned parameters, but also taking into account uncertainty over the parameters using a prior probability distribution. As we explain, adding this prior allows us to treat learning as Bayesian inference, which is highly useful for online learning.

Let $M_{pa_l,l}$ be the matrix of learned conditional probabilities over node $l$ given its parent. The values of the parent nodes are discrete integer values represented, in SQHN networks, by the one-hot $h^*_{pa_l}$. The conditional distribution according to the learned matrix is $p_l = M_{pa_l,l} h^*_{pa_l}$, which is equivalent to the column of $M_{pa_l}$ indexed by $h^*_{pa_l}$. However, $p_l$ does not take into account the uncertainty over our parameters. Accounting for uncertainty is important especially early in online training, since early in training parameters have been updated using only a small fraction of the data set, and therefore these parameters are less 'trustworthy' than 

A common method to represent uncertainty over parameters is to treat parameters as a random variable and place a prior distribution over it. Learning then amounts to performing Bayesian inference, where the maximum likelihood (learned) parameters are combined with the prior distribution. The typical prior distribution over parameters for discrete graphical models, like Bayesian networks, are Dirichlet priors, which are the conjugate prior for discrete (e.g., categorical and multinomial) distributions. Details of about this prior can be found in \cite{heckerman1998tutorial}. Here we simply point out that the most common Dirichlet prior is the uniform distribution. If node $l$ is a discrete distribution then the distribution the represents no prior knowledge is the uniform distribution. 

There are several ways to place this prior over parameters. First, a prior distribution may be place over each individual column of $M_{pa_l,l}$, where each column represents the conditional distribution for a different parent node value. In this case, uncertainty about the distribution encoded in each column of $M_{pa_l,l}$ is represented separately, and may differ between columns. We use a simpler approach, which is to have one measure of uncertainty over $M_{pa_l,l}$ as a whole. Let's call $\theta^0$ the parameters that set each column of each matrix equal to the uniform distribution. In particular, if we assume that the parent of $l$ is assigned value $j$, we compute the conditional probability distribution as 
\begin{equation}
p(h^{*}_l| pa_l, \theta^0, \alpha) = \frac{\alpha \frac{1}{J_l} + t p(h^*_l | pa_l)}{t + \alpha} = \frac{\epsilon}{J_l} + (1 - \epsilon ) p(h^*_l | pa_l),
\end{equation}
where $t$ is the total number of data points observed so far and $\epsilon = \frac{\alpha}{t + \alpha}$. The value $J_l$ is the number of values node $l$ can take and thus $\frac{1}{J_l}$ represents the prior (uniform) probability over child node values. The probability $p(h^*_l | pa_l)$ is the learned conditional probability of $h^*_l$ generated by $M_{pa_l, l}$. The conditional probability at iteration $t$ then is a weighted average between the uniform distribution and the learned distribution generated by $M_{pa_l, l}$. The learned distribution is weighted increasingly heavily as the number of data points observed, $t$, increases. The weighting also depends on the hyper-parameter $\alpha$: the larger $\alpha$ is the more heavily the uniform distribution is weighted and the slower its influence will decay.

Taking the joint of these conditionals we get
\begin{equation}
p(h^*_0, h^*_1,...,h^*_L) = \prod_{l=0}^L p(h^{*}_l| pa_l, \theta^0, \alpha) = \prod_{l=0}^L (\frac{\epsilon}{J_l} + (1-\epsilon) p(h^*_l | pa_l)).
\end{equation}
We derive an energy from this expression by approximating the maximization of the joint joint distribution. It is approximated by expanding the product and removing certain terms that are guaranteed to be small relative to other terms and by removing constants which do not affect the local maximum of the joint. To simplify notation, without loss of generality, assume all nodes have the same number of possible values $J$. Further, let $g = \frac{\epsilon}{J}$ and $p(h^*_l | pa_l) = p_l$ and $L$ be the number of nodes. Now if we expand the equation above we get a combination of sums and products of probabilities:
\begin{equation}
\begin{split}
\prod_{l=0}^L (g + (1-\epsilon) p_l) =  g^L &+ g^{L-1} (1 - \epsilon) (\sum_{l=1}^L p_l)\\
&+ \sum_{l=2}^L g^{L-l}(1-\epsilon)^l (\prod_{k=1}^{l} p_k) + G\\
=  c_0 + &c_1 \sum_{l=1}^L p_l + \sum_{l=2}^L c_2^l \prod_{k=1}^{l} p_k + G\\
\end{split}
\end{equation}
where $c_0$, $c_1$, $c_2$ are scalar coefficients, and G is a placeholder for a large number of similar terms we were unable to express concisely. These terms are similar in the sense they all involve a scalar $x$, a sum of some number k probabilities $y = \sum^k_{l=1} p_l$ multiplied by by a product of probabilities $z = \sum^{L-n}_{k} p_k$, where $n$ is some number less than $L$, having the form $xyz$.

So we have four terms. A constant scalar $c_0$, a sum of the conditional probabilities $c_1 \sum_{l=0}^L p_l$, a product of probabilities, and a large sum of terms that multiply a sum of probabilities by a product of probabilities. Let's collapse $G$ and the product term into the term $S$. In the limit where $S \rightarrow 0$ we have
\begin{equation}
\begin{split}
\argmax_{h, \theta} p(h^*_0, h^*_1,...,h^*_L) &= \lim_{S\to0}  \argmax_{h, \theta} c_0 + c_1 \sum_{l=0}^L p_l + S\\
&= \argmax_{h, \theta} c_0 + c_1 \sum_{l=0}^L p_l \\
&= \argmax_{h, \theta} \sum_{l=0}^L p_l\\
&= \argmax_{h, \theta} E.
\end{split}
\end{equation}
Thus, under this limit $E \approx p(h^*_0, h^*_1,...,h^*_L)$. 

What does it take for this limit to be well approximated in practice? Each in $S$ has form $xyz$, where $x$ is a non-zero scalar, $y$ will typically be non-zero since it is a sum of probabilities, and z is a product of probabilities, which will be zero when at least one term in the product is zeroed out. The more terms in the product $z$ the more likely the term will be zero, and the less terms the less likely it will be zero (assuming roughly equal probability any one term will be zero). When $z$ has few terms, $y$ is a sum of many probabilities (see above). Thus, those terms in $S$ which have a product $z$ of many probabilities and a sum $y$ of few will be zeroed out more often on average than terms with products of a small number of probabilities and sums of many. This suggests in networks where it is often the case that multiple conditional probabilities are zero each iteration, a sum of probabilities is a good approximation of the joint

SQHNs generally have highly sparse weights, and therefore it will often be the case that multiple conditional probabilities will be zero during learning. For example, if a node $l$ grows a new neuron at the current iteration, its conditional probability according to the parameters at the current iteration will be zero. If node $l$ takes a value that was not previously observed in combination with its sibling node values, it (or at least one of its siblings) will have probability 0. \textit{This suggests that maximizing our energy function in SQHNs is a decent approximation to maximizing the joint with the Dirichlet prior over parameters as described above.}

\subsection{Derivation of Inference Procedure}\label{supp:infer}

MAP inference in tree-structured graphs with discrete nodes can be implemented via the max-product algorithm \cite{bishop2006pattern}, whose goal can be described as:
\begin{equation}
\begin{split}
\textbf{Max-Product: } &\argmax_{h^*} \prod_{h^*_l} p(h^{*}_l|pa_l^t),\\ 
\end{split}
\end{equation}
where $h^*$ is the set of integer value assignments for each node. The max-product algorithm works by performing a single FF and FB sweep through the network. Consider a tree structured graph where, without loss of generality, we assume the simple case where nodes take one of two value. In this case, the conditional probability of node $l$ given its parent node are represented by matrix $M_{l,pa_l} = [m_0, m_1]$. The max-product algorithm first propagates a signal from the visible nodes up through hidden nodes to the root node using the operation (in our notation) $h_{pa_l} = [\max(m_0 \otimes h_{l}), \max(m_1 \otimes h_{l})]$ \cite{bishop2006pattern}, where $\max$ (not to be confused with the activation function $max$) outputs a single scalar. In the case of multiple children nodes inputs to $h_l$ are multiplied element-wise. The root node value is set equal to this maximum value, $h^*_L = \argmax h_L$, then an operation known as backtracking is performed, where a signal is propagated back from the root node down the tree to find the MAP values for hidden nodes (see \cite{bishop2006pattern} for details). An analogous algorithm may be used to maximize the SQHN energy $E$, which is a sum rather than product of conditional probabilities. The analogous operation is $h_2 = [\max(m_0 + h_1), \max(m_1 + h_1)]$. These operations may be used in the same procedure as the max-product to obtain the max $E$ value at the root node, MAP hidden node values via a backtracking procedure.

The issue with this approach is that the main operation used in the FF sweep is not a standard vector matrix multiply. Instead of multiplying each row/memory vector of $M_2^{\top}$ element-wise by $h_1$ then summing over elements, as in a matrix multiply, each row and $h_1$ are multiplied element-wise then the max value is returned. However, our goal is to create a neural network, that may be easily implemented in hardware that assumes vector matrix multiplies as the main operation (e.g., GPUs and memoristor based neuromorphic hardware). Therefore, we propose the following alternative approximate MAP inference procedure, which only uses vector matrix multiplies and a neuron-wise normalization operation. Like the MAP inference procedures described above, this procedure involves a feed forward (FF) and feedback sweep through the network. During the FF sweep, each node, starting at the lowest layer and working up, updates according to:
\begin{equation}\label{eq:suppInptHid}
h_l = \frac{1}{Z}\sum^C_{c \in ch(h_l)}M_{l,c}^{\top} max(h_c),
\end{equation}
where $ch(l)$ is the set of nodes that are children of $h_l$ and $M_{l,c}$ is the matrix from node $l$ to child node $c$ and the italicized $max$ operation to express an activation function that output a vector of all zeros, except for the max element: e.g., $max([.3, .8, .4]) = [0, .8, 0]$. The summed input is normalized by $Z = \frac{1}{C \sqrt{\sum_c \max(h_c)^2}}$.

We can also describe the activation at $l$'s child node as $max(h_c) = argmax(h_c) \max(h_c) = h^*_c \max(h_c)$, where $argmax$ is the activation that returns a one-hot, and the un-italized $\max$ returns a scalar, e.g., $\max([.3, .8, .4]) = .8$. Further, for each each neuron $j$ in node $l$ it is the case that $M_{l,c,j}^{\top} h^*_c = p(h^*_c | h_l=j)$. This means we can express the input to each neuron $j$ in node $l$ as 
\begin{equation}\label{eq:wgtAvgChild}
h_{l,j} = \frac{1}{Z}\sum^C_{c \in ch(h_l)} p(h^*_c | h_l=j) \max(h_c).
\end{equation}
Thus, the input to each neuron $j$, during the FF sweep is a sum of the probabilities of child nodes, given node $l$ has value $j$, where each probability is weighted by the associated max value at the child node. Therefore, $h_{l,j}$ is the \textit{weighted energy of $l$'s children node values}, when $l$ takes value $j$. The term $Z$ then takes the average, so that $h_{l,j}$ is a weighted average of the probabilities of child nodes given node $l$ has value $j$. \textit{Thus, the FF sweep sets nodes to the value that maximizes the weighted average of the probability (i.e., weight energy) of child node values.}.

\begin{wrapfigure}{r}{0.5\textwidth}
  \begin{center}
  \includegraphics[width=0.49\textwidth]{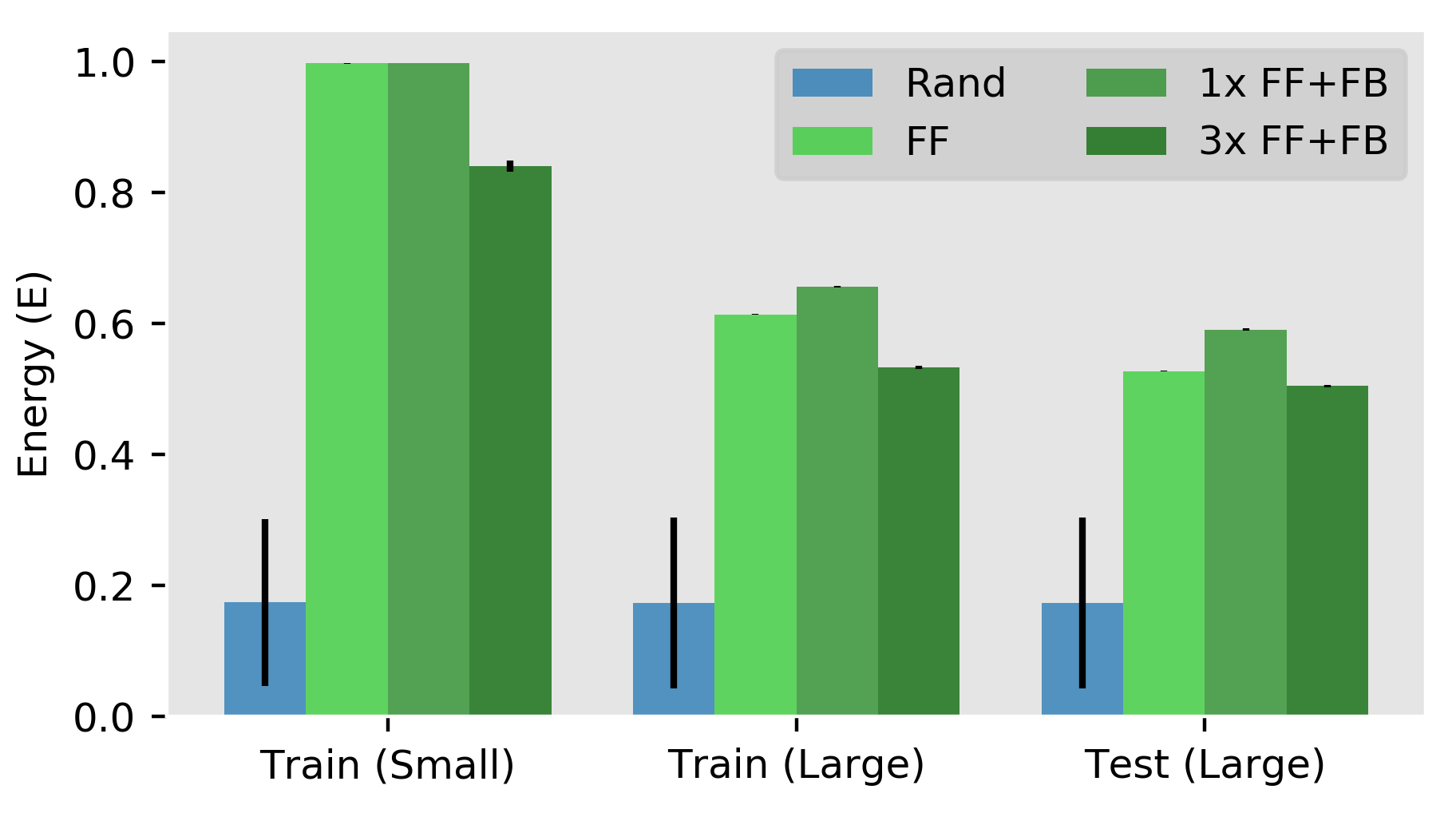}
  \end{center}
  \caption{Energy measurements during SQHN inference procedure. Energy after a single feed forward (FF) sweep, one FF and feedback sweep (1x FF+FB), and three FF FB sweep (3x FF+FB). Energy shown on small train set (200 images), and on large train set (1000 images) and test/hold-out data set. One FB sweep increases energy, but the network performs best only after a single FF and FB sweep.}
\end{wrapfigure}\label{fig:energy}

Importantly, the weighting terms carry information about the energy of descendent nodes: the weighting term $\max(h_c)$ can be defined recursively via equation \ref{eq:wgtAvgChild}, which shows that $h_c$ is itself a weighted average of its children node probabilities (i.e., grandchildren of $l$), and so on. \textit{Therefore, weighting terms act as a kind of bottom up attention mechanism, where inputs from children are weighted based on the weighted energy/probability of grandchildren, whose weight are based on the energy/probability of great-grandchildren, and so on.} 

Nodes at the first hidden layer, on the other hand, receive a signal from visible nodes with no children. It turns out this signal can still be interpreted as the a weighted average of the likelihood of each pixel value. The input to a node at the first hidden layer is
\begin{equation}\label{eq:suppInptBot}
h_{l,j} = \frac{.5 (m_{l,j}^{\top} - .5) (x_{c_l} - .5)}{\Vert (m_{l,j}^{\top} - .5) \Vert \Vert (x_{c_l} - .5) \Vert} + .5,
\end{equation}
where $x_{c_l}$ is a vector of input value (e.g., from an image patch). The input we use has values between 0 and 1. We treat each value as a binary variable whose probability is 
\begin{equation}\label{eq:llhood}
\begin{split}
p(x_{c_l,i} | h^*_l) &= x_{c_l,i} p_{c_l,i} + (1 - x_{c_l,i}) (1 - p_{c_l,i})\\
&= 2 x_{c_l,i} p_{c_l,i} - x_{c_l,i} - p_{c_l,i} + 1
\end{split}
\end{equation}
Next, consider that the prediction $p_{c_l}$ just equals the memory vector $m_{l,j}$ indexed by $h^*_l$. Given this, we can rewrite the numerator of equation \ref{eq:suppInptBot} as
\begin{equation}
\begin{split}
 (m_{l,j,i} - .5) (x_{c_l,i} - .5) &= (p_{c_l,i} - .5) (x_{c_l,i} - .5) \\ 
 &= p_{c_l,i} x_{c_l,i} - .5 p_{c_l,i} - .5 x_{c_l,i} + .25 \\ 
 &= .5 p(x_{c_l,i} | h^*_l) - .25,\\ 
\end{split}
\end{equation}
which is just the liklihood of the input with an affine shift applied.
Thus, the numerator in equation \ref{eq:suppInptBot} computes an affine shifted version of the likelihood of each pixel and sums these together. The denominator ensures the output is between -1 and 1, which is rescaled to the range 0 and 1 by the multiplication and shift of .5. The advantage of computing the likelihood this way, rather than a direct computation of equation \ref{eq:llhood} is that this operation can be implemented via a neural network like matrix vector multiply with neuron-wise normalization:
\begin{equation}
h_{l} = \frac{1}{2 Z} (M_l^{\top} - .5)(x_{c_l} - .5) + .5,
\end{equation}
where $Z$ is a vector of the terms computed accord to the denominator in equation \ref{eq:suppInptBot}. It is also possible to just store a separate feedback matrix, with shifted and normalized rows.

Finally, after activities $h_l$ are computed via the FF sweep, the final $h_l$ values are computed using the top-down/FB sweep. Let $h_{l+n}$ be the parent of $h_l$
\begin{equation}\label{eq:post}
h_l = (\lambda) h_l + (1 - \lambda) p_l,
\end{equation}
where $p_l = M_{pa_l,l} h^*_{pa_l}$ and $\lambda$ is a scalar between 0 and 1 that modulates the influence of the top-down signal, which we typically set to $.5$ during recall.

\begin{algorithm}[H]
\SetAlgoLined
\DontPrintSemicolon
\Begin{
    \For{$t=1$ \KwTo $T$}{
        \tcp{Clamp visible nodes to $x^t$}
        \tcp{Inference}
        \For{$l=0$ \KwTo $L$}{
            \tcp{Compute bottom up input $h_l$, equation \ref{eq:suppInptBot}, \ref{eq:suppInptHid}}
            }
        \For{reversed($l=0$ \KwTo $L-1$)}{
            \tcp{Combine $h_l$ and top-down input, equation \ref{eq:post}}
            \tcp{Set $h^*_l = argmax(h_l)$}}
        \tcp{Set each input patch $x_{c_l} = M h^*_l$}
    }
}
\caption{SQHN Recall Algorithm \label{alg:recall}}
\end{algorithm}

\subsection{Neuron Growth and the Dirichlet Process Prior}
The inference process described above assumes a fixed number of neurons at each node. It is assumed neuron numbers are fixed during recall, when we are performing inference over \textit{old} data points, from the training set. However, intuitively, it may be desirable to add new neurons and synapses dynamically as needed during learning, when the model performs inference over \textit{new} data points. A common principled way to dynamically add new components to a discrete/categorical distribution is via the Dirichlet process prior. In parametric models with discrete variables (e.g., standard mixture models), the number of components $J$ is treated as a hyper-parameter, which set by the modeler. In non-parameteric models, $J$ is learned/inferred by the model according to a prior distribution over $J$. The Dirichlet process prior (DPP) is a common prior for discrete non-parameteric models. We do not go into details on the derivation of the DPP prior here. We only explain how it inspires our model. For good tutorials see \cite{gershman2012tutorial, li2019tutorial}.

The DPP provides a method for determining when it is more probable that some input belongs to a new latent, integer value ($J+1$), rather than an already existing value ($\leq J$). It does this by representing the prior probability that the input belongs to a new integer value, given all previous assignments: $p(h_{l}^{*,t}=new | h_{l}^{*,0:t-1})$, which is the prior probability the input to node $l$ at iteration $t$ belongs to a new value given all of the previous value assignments up until current time $t$. The DPP sets this prior probability as
\begin{equation}
\begin{split}
p(h_{l}^{*,t}=new | h_{l}^{*,0:t-1}, \alpha) &= \frac{\alpha}{t + 1 + \alpha}\\
\end{split}
\end{equation}
where $\alpha$ is a hyper-parameter. The DPP is order invariant \cite{gershman2012tutorial}, which means that the ordering the previous data points arrive does not affect the computation of the prior. The posterior probability that the node value is new, multiplies the DPP above with the likelihood of the input according to some prior distribution over parameters \cite{li2019tutorial}. This posterior can then be compared to the posterior probabilities of existing values given a data point, $x$, and existing parameters. One can then make the decision to add a new value, if this is the most probable case.

Since the SQHN maximizes an energy function rather than the actual joint distribution and performs approximate inference rather than exact MAP inferance over the energy, our algorithm for neuron growth is inspired by the DPP rather than an exact implementation of it. In particular, we use the following simple computation as an estimate of the energy value associated with adding a new neuron:
\begin{equation}\label{eq:suppDir}
\epsilon = \frac{\gamma \alpha}{t + 1 + \alpha},
\end{equation}
where $\alpha$ and $\gamma$ are hyper-parameters and $0 < \gamma \leq 1$ and $0 < \alpha$. One can think about $\gamma$ as analogous to an estimate of the likelihood the input belongs to a new value (according to some prior distribution) and $\frac{\alpha}{t + 1 + \alpha}$ as analogous to the prior probability of the the input belonging to a new node value/neuron. During learning the SQHN performs inference before updating weights. If the maximum internal state value (energy), $\text{max}(h_l)$, at node $l$ is less than $\epsilon$, a new neuron is grown. Otherwise, the node is set to the maximum existing value (see next section).

\subsection{Derivation of Learning Algorithm}\label{supp:learn}
Under the SQHN algorithm, parameters are updated each iteration to solve the following optimization problem:
\begin{equation}
\begin{split}
\theta^t &= \argmax_\theta \sum^t_{i=0} E(\theta, h^{*,i}, x^i)\\
&\text{s.t.} \sum_j p_{l,j} = 1 \text{ } \forall l^{hid} \text{  and } p_{l,j} \geq 0 \text{ }  \forall l,j\\
\end{split}
\end{equation}
where $h^{*,i}$ are one-hot value assignments for each node associated with data point $x^i$, $t$ is the current training iteration, $x^i$ is the data point presented at iteration $i$, and $l^{hid}$ refers to hidden nodes. The weight update at iteration $t$ sets parameters equal to the values that maximize the energy over the current and all previously observed data points, and associative hidden states, under the constraint the prediction $p_l$ at hidden nodes are properly normalized. Importantly, since this update is computed online, we assume we have to perform this update only provided the data point and hidden states at the current iteration (i.e., we are assuming there is no buffer storing previous data and hidden states). This leaves two questions: 1) How do we solve this constrained optimization problem, in the online setting, and 2) how should the values of $h^*$ be set each iteration?

The same optimization above has been solved for the Bayesian networks \cite{heckerman1998tutorial}. Bayesian networks update parameters to maximize the \textit{product} of conditional probabilities. SQHN networks update parameters to maximize the \textit{sum} of conditional probabilities. Despite these differences, the goal ends up being equivalent w.r.t. to parameter updates: each matrix representing the conditional probabilities between parents and children must be updated to maximize the conditional probabilities of the observed child values (under normalization constraints). Thus, SQHN networks can utilize the solution for Bayesian networks as a basis for its online learning rule.

Consider the case where child node, $l$, is a discrete/categorical variable. The matrix containing the conditional probabilities between the parent of $l$, $pa_l$ and $l$ is $M_{l, pa_l}$. The columns of this matrix contain the conditional probability distributions given various values of the parent node: $m_{l, pa_l} = [m_{l,0}, m_{l,1}...,m_{l,J}]$, where $m_{l,j}$ is the conditional probability distribution over values of $l$, given $pa_l = j$. Each value of $m_{l,j}$ tells us the probability that $l$ takes a particular value given the $pa_l = j$. In particular, $p(l=k | pa_l=j) = m_{l,j,k}$, where $k$ is the kth value of $m_{l,j}$. The update for Bayesian networks tells us that this value is simply the number of times $j$ and $k$ co-occurred previously, divided by the total number of times parent to value $j$ occurred \cite{heckerman1998tutorial}:
\begin{equation}
p(l=k | pa_l=j) = m_{l,j,k} = \frac{c_{l,j,k}}{c_{l,j}},
\end{equation}
where $c_{l,j,k}$ is the number of times $pa_l=j$ and $l=k$ at the same iteration. The term $c_{l,j}$ is the number of times $pa_l=j$. This solution entails the matrix column $m_{l,j}$ is an average over the one-hot values of the child node that were present when $pa_l=j$:
\begin{equation}
m_{l,j}^t = \frac{1}{c_{l,j}^t} \sum_{i\in t_j} h^{*,i}_l,
\end{equation}
where $t_j$ refers to the set of time steps where $pa_l = j$. This same average can be computed online:
\begin{equation}
m_{l,j}^t = m_{l,j}^{t-1} + \frac{1}{c_{l,j}^t} (h^{*,t}_{l} - m_{l,j}^{t-1}),
\end{equation}
under the assumption $pa_l = j$. This same update can more generally be implemented as a local, Hebbian-like learning rule:
\begin{equation}\label{eq:MUpdate}
\begin{split}
M_l^t &= M_l^{t-1} + \frac{1}{c_{l}^{t,\top} h^*_{pa_l}} (h^{*,t}_{l} - p_{l}^{t-1}) h_{pa_l}^{*,t \top}\\
\end{split}
\end{equation}
where the update uses the fact that $p_l = M_l h^*_{pa_l}$ which trivially entails that $p_l$ equals the memory vector indexed by $h^*_{pa_l}$, e.g., if $pa_l=j$ then $p_l = m_{l, j}$. Second, the step size $\frac{1}{c_{l}^{t,\top} h^*_{pa_l}}$ uses the fact that $h^*_{pa_l}$ indexes the correct count, e.g. if $pa_l=j$ then $c_{l}^{t,\top} h^*_{pa_l} = c^t_{l, j}$.

This is the update rule if children nodes are are discrete/categorical variables. We use the same update for visible nodes, which we treat as vectors of binary variables. The properties remain the same in this case: $m_{l,j,k}$ ends up representing an average over previously observed values of input value $k$ in image patch $l$.

Next, is the question of how to set the hidden node values during training? The weight updates above assume $h^*_l$ values are \textit{observed}. However, when $h^*_l$ is a hidden node its values are, by definition, not observed. One option for setting hidden node values is to use MAP learning. MAP learning works by first performing MAP inference over hidden variables to find the specific values of the hidden variables, $h^*$, that maximize the posterior $P(h^* | x, \theta) \propto P(h^*, x, \theta)$. Parameters are then updated to further increase the probability of the joint $P(h^*, x, \theta)$. We use a variant of MAP learning here. In particular, hidden node values are determined by an inference procedure where each node maximizes the probability of its children nodes: each node receives the FF signal from its children, then it is either assigned a one-hot indexing the existing neuron with the maximum internal state value, or if all existing neuron values are below the threshold $\epsilon$, the node is set to a new value (i.e., a new neuron is grown) (see algorithm \ref{alg:learn}). 

The resulting algorithm as a kind of hybrid between MAP learning and learning in the case where all node values are observed. Like MAP learning we set the hidden variables to a point estimate value with the highest probability. (This is opposed to computing a distribution over hidden nodes values, which is what the common expectation maximization \cite{dempster1977maximum, heckerman1998tutorial}) algorithm does.) However, unlike MAP learning, we do not compute the maximum of the posterior since we ignore FB/top-down signal in inference. Instead, similar to \cite{george2017generative}, our algorithm computes the values that maximize the conditional probability of its child node values, while ignoring parent node values. Nodes, therefore, maximize the likelihood of their child node values. Thus, like the case where all node values are observed, each parent node does not influence/determine the value of its child node, but instead only 'observes' its child's values and updates its conditional distribution accordingly. But unlike this fully observed case, hidden node values are set via inference.

We find this max-likelihood variant works better than MAP in the online setting. MAP inference ends up pushing hidden node values to the values that were inferred in previous iterations, making the model get 'stuck' reproducing the same hidden variable values rather than learning new combinations of values for different inputs. Removing the FB/top-down influence during learning allows nodes to assign new sets of values to new inputs, as desired.
\begin{algorithm}[H]
\SetAlgoLined
\DontPrintSemicolon
\Begin{
    \For{$t=1$ \KwTo $T$}{
        \tcp{Clamp visible nodes to $x^t$}
        \tcp{Max Likelihood Inference}
        \For{$l=0$ \KwTo $L$}{
            \tcp{Compute bottom up input $h_l$, equation \ref{eq:suppInptHid}, \ref{eq:suppInptBot}}
            \tcp{If $h_l$ is less than $\epsilon$, set $h^*_l$ to new value.\\Else $h^*_l = argmax(h_l)$}}
        \tcp{Update $\epsilon$, equation \ref{eq:suppDir}}
        \tcp{Update Weights, equation \ref{eq:MUpdate}}
    }
}
\caption{SQHN Learning Algorithm \label{alg:learn}}
\end{algorithm}

\subsection{Derivation of Episodic Recognition Rule}\label{supp:recog}
Recognition tasks begin by presenting a sequence of items/data points from the training set $X_{train}$ during a training phase. Then during a testing phase, the model is presented with a mixture of old data points from $X_{train}$ and new data points from a similar (in-distribution) data set $X_{in-dist.}$ and from a dissimilar (out of distribution) data set $X_{out-dist.}$. An equal portion of data is presented from each data set during testing. The model must correctly judge if the data point is in the training set (old) or not (new). This task is based on classic and common tests of human memory \cite{rugg2003human}.

Importantly, episodic recognition is distinct from the common OOD detection task in machine learning and will therefore require a different solution. OOD detection is the task of detecting data drawn from a distribution distinct from the training set \cite{yang2021generalized}. Solutions typically involve learning a generative model of the training data, then computing the likelihood of data under this model \cite{yang2021generalized}. This is not the same as detecting data that was not present during training, since many unobserved/new data may still be from sampled from the same distribution as the training set, yet still be unobserved/new. 

Inspired by probabilistic models of human episodic recognition (e.g. \cite{shiffrin1997model}), we instead use the following approach. We assume each training input $x^t$ is mapped to, what we call, a global feature representation $h^*$. By global, we mean it represents features of the entire input (e.g., entire image) rather than just a sub-portion of it (e.g. image patch). Then each feature vector is stored in an itemized memory via a mixture model. Let $M$ be the matrix that contains cluster means. Feature representations, $h^{*,t}$ at each training iteration $t$, get stored in the columns of $M$, where $M = [h^{*,0}, h^{*,1},...]$. If every data point from the train set gets stored separately then during testing the model can simply compare a feature representation $h^*$ of the input to the stored feature representations. If the likelihood is around 1, then the model judges 'old' if less then one, it judges new.

Up until it reaches capacity, and with a high $\alpha$, this is exactly what happens at the memory node of SQHN. The memory node has full receptive field and takes as input the activities from the values of its children, which represent the input's visual features. Thus, until capacity is reached, SQHN has a clear way of making old/new judgments: perform max-likelihood inference (see algorithm \ref{alg:learn}), and check if children of the memory node have probability/energy $\approx 1$. 

However, after capacity is reached, every data point is not stored in $M_L$ in its original form. Instead, new inputs will be averaged with old ones. This raises the question of how judgments should be make after capacity is reached. We use the following strategy. Let $p(h^*_{c_L} | h^*_L=j)$ be the likelihood of the values of the children of the memory node given the memory node value $h^*_L=j$. Here for simplicity we assume this is a good estimate of the likelihood $p(x^t | h^*_L=j)$. Let's say $n$ data points from the train set have been assigned to $j$. Lets say we know the data point assigned to this value with the lowest likelihood is $x_{min}$. If we know this likelihood, we know any new data point assigned to $j$ that has a likelihood lower than this minimum must be new (i.e., have 0 probability of being old). Thus, we propose approximating the probability that some data point $x^t$ is old using the re-scaled likelihood 
\begin{equation}
p(x^t=old | h^*_L=j) \approx p(x^t | h^*_L=j) * (1 - p(x_{min} | h^*_L=j)) + p(x_{min} | h^*_L=j) 
\end{equation}
Then to make a recognition judgments, the model checks if $p(x^t=old | h^*_L=j)$ is greater than or less than .5. Equivalently, we could set a threshold, $\mu$, equal to the likelihood value at which $p(x^t=old | h^*_L=j) = .5$. In practice, we find the minimum and midpoint likelihood difficult to exactly compute online. Instead, we use a simple estimate of the mid-point value, which we find works well in practice. This estimate just keeps a moving average of the likelihoods:
\begin{equation}
\mu_{l,j}^t =  \frac{1}{c^t_{l,j}} \sum_{n=1}^t h^{*,n}_{l,j} = \frac{c^t_{l,j}-1}{c^t_{l,j}}\mu_{l,j}^{t-1} + \frac{1}{c^t_{l,j}} h_{l,j}^{*,t},
\end{equation}
where the equation on the right is how to compute this average online.

\begin{algorithm}[H]
\SetAlgoLined
\DontPrintSemicolon
\Begin{
    \For{$t=1$ \KwTo $T$}{
        \tcp{Clamp visible nodes to $x^t$}
        \tcp{Max Likelihood Inference (w/o Neuron Growth)}
        \For{$l=0$ \KwTo $L$}{
            \tcp{Compute bottom up input $h_l$, equation \ref{eq:suppInptBot}, \ref{eq:suppInptHid}}
            $h^*_l = argmax(h_l)$}
        \tcp{If $\max(h_L) > \mu_j$ (where $\argmax(h_L) = j$), judge old}
        \tcp{Else judge new}
    }
}
\caption{SQHN Episodic Recognition Algorithm \label{alg:recogn}}
\end{algorithm}

\subsection{Theoretical Results}\label{supp:theorResults}
\textbf{Capacity} We analyze the properties of the recall accuracy of an SQHN with one hidden layer. We define the capacity of the network as the maximum number of data points the network is able to recall given some recall threshold $\gamma$. Let $J$ be the number of neurons at the hidden layer.
\begin{theorem}\label{thrm:capacity}
Assume no two data points in the training data set are scalar multiples of each other and during recall data points are not corrupted. The capacity of a single hidden layer SQHN network is at least J, i.e., the number of neurons at the hidden layer, for any $\gamma \geq 0$.
\end{theorem}
\begin{proof}
Assume a single data point is stored in each memory vector. If the child node is a discrete variable then the input to each hidden neuron, $j$ is the cosine similarity between the input and the jth memory vector. If the child node is a binary variable then the input to each hidden neuron, $j$ is the mean-shifted cosine similarity (equation \ref{eq:inpLayerBot}) between the input and the jth memory vector. Both will return the maximum possible value for two identical vectors, as long as no two vectors have the same angle. We assume no to vectors have the same angle (no two stored vectors are scalar multiples of each other). Therefore, the maximum value at the hidden layer is guaranteed to index the correct stored training data point, which will yield perfectly accurate MSE of 0 and thus perfect recall accuracy for any $\gamma$. If some memory vectors average over data points, it is not guaranteed (though possible depending on how small $\gamma$ and how similar averaged vectors) returned memory vectors yield correct recall for obvious reason. Thus, the minimum capacity is at least the number of neurons at the hidden layer, J.
\end{proof}

\textbf{Forgetting} It is also important to understand what happens when the SQHN network is pushed passed capacity. Ideally, memory models should show 'graceful' forgetting in such cases, where performance decreases in a predictable and non-abrupt manner. We assume a kind of  worst-case conditions where 1) data points are highly orthogonal such that averaging two or more data points in some memory vector $m_j$ means that returning the memory vector $m_j$ during recall does meet the condition for successful recall for any of the data points it was averaged over, and 2) each existing memory vector is equally likely to be accessed and updated during training.
\begin{theorem}\label{thrm:forget}
Under the worst case assumptions, the number of memory vectors that store a single data point decreases at an exponential rate of $J e^{\frac{-t}{J}}$, where $t$ is the training iteration (starting after the network is at capacity) and $J$ is the number of hidden units. This entails its recall accuracy will decay at rate $\frac{J e^{\frac{-t}{J}}}{J+t}$.
\end{theorem}

\begin{proof}
The process we are characterizing is the training of a basic SQHN unit that begins at capacity (each memory vector is set equal to one previously observed data point). Each training iteration, $t$, a memory vector is chosen with uniform probability and the data point observed at that iteration is averaged with the chosen memory vector. There are thus two kinds of memory vectors: those that store a single data point and those that store an average over more than one data point. We want to characterize how the number of memory vectors that store only a single data point decrease over time on average, where by 'on average' we mean that the decreases are measured and averaged over an infinite number of training runs.

Let $I(t)$ be the number of memory vectors, at time $t$, that store only a single data point. Under these assumptions, the decrease in $I(t)$ is clearly equal to the probability of choosing and updating a memory vector that store only a single data point, since this is equal to the proportion of training runs that will choose and update such a memory vector. Under the assumption of a uniform probability distribution, the probability of choosing and updating a memory vector that equals a single data point is just equal to the proportion of memories that store a single data point at time $t$:
\begin{equation}
\frac{\partial I(t)}{\partial t} = - \frac{I(t)}{N}.
\end{equation}\label{eq:forgetRate}

A function $f(x)$ decays exponentially according to $f(x_0) e^{-\lambda x}$ with rate $\lambda$ and initial value $f(x_0)$, if $\frac{\partial f(x)}{\partial x} = - \lambda f(x)$. Clearly, the rate of forgetting shown above is an exponential with rate $\frac{1}{N}$. At the initial iteration it is assumed $I(0) = N$. Therefore, the forget rate can be described as $N e^{\frac{-t}{N}}$.
\end{proof}

Thus, the SQHN memory unit or one hidden layer model shows a graceful, exponentially decaying forget rate.

\textbf{Learning and Parameter Isolation} Let's consider a distance metric between the parameters before and after the update in terms of the amount of energy that is left for parameters to reduce: 
\begin{equation}
d(M_l^{new}, M_l^{old}) = E_{ch_l}(M^{new}_l, h_l^*, h_{ch_l}^*) - E_{ch_l}(M^{old}_l, h_l^*, h_{ch_l}^*),
\end{equation}
where $E_{ch_l}$  is the energy (averaged conditional probability) of the children of $l$, and $M_l^{old}$ is the matrix of conditional probabilities from $l$ to its children before the weight update, and $M_l^{new}$ the conditional probabilities after. Thus, this distance measure tells us how different the conditional probabilities of the children are under the new and old matrix.

Here we analyze properties of the SQHN learning algorithm in a case where two simplifying assumptions are made: 
\begin{enumerate}
\item Each column of $M_l$ is updated with an equal, constant step size.

\item The weight are updated to maximize energy to a specific value: $\kappa = E_{ch_l}(M^{new}_l, h_l^*, h_{ch_l}^*)$. 
\end{enumerate}
From this the following theorem clearly follows.

\begin{theorem}
Under the assumptions above, the value of $h_l^*$ computed during the training phase of the SQHN is equivalent to 
\begin{equation}
h^*_l = \argmin_{h_l^*} d(M_l^{new}, M_l^{old}).
\end{equation}
\end{theorem}

\begin{proof}
During learning, each node $l$ takes value $h^*_l$ that maximizes the energy of its children. Thus, during learning $h^*_l$ is the value that maximizes $E_{ch_l}(M^{old}_l, h_l^*, h_{ch_l}^*)$. From assumption 2, we know after the update for any value of $h^*_l$, the energy of the children is $\kappa$. Thus, it clearly follows that the value $h^*_l$ can be described as the value minimizing 
\begin{equation}
\begin{split}
\kappa - E_{ch_l}(M^{old}_l, h_l^*, h_{ch_l}^*) &= E_{ch_l}(M^{new}_l, h_l^*, h_{ch_l}^*) - E_{ch_l}(M^{old}_l, h_l^*, h_{ch_l}^*)\\ &= d(M_l^{new}, M_l^{old}).
\end{split}
\end{equation}
\end{proof}
In this simplified case, setting $h^*_l$ to the value that maximizes the energy of the children of $l$, is equivalent to setting $h^*_l$  equal to the value that minimizes \textit{the amount that conditional probabilities encoded in $M_l$ change during the weight update}. By first maximizing the energy w.r.t. $h^*_l$, there is less 'work' for weight updates to do when they increase energy w.r.t. $M_l$. In other words, the SQHN algorithm seems to have a principled method for choosing which parameters to isolation during training, i.e., update the column in $M_l$ that requires the least amount of change given the value of children nodes. This makes sense from the point of view of prevent catastrophic forgetting, where we generally want to reduce the amount parameters change. Future work could look to understand how this result applies to less simplified cases, better describing the SQHN.

\subsection{Relations to Hopfield Network and Transformers}\label{supp:SQHNvHop} The SQHN has close relations to the continuous modern Hopfield network (MHN) \cite{ramsauer2020hopfield}. The continuous MHN network is
\begin{equation}\label{eq:cMHN}
x_{new} = M softmax(\beta M^T x),
\end{equation}
where $M$ is a matrix of data points/memory vectors, $x$ is an input/query vectors, and $\beta$ is the temperature. The MHN takes the dot product between each memory vector and the query vector and passes the output through a softmax softmax. These softmax values are then multiplied again by the weight matrix. It was shown by \cite{ramsauer2020hopfield} that this model is identical to the attention layer of the transformer in the case where the transformer key (K) and value (V) matrices are tied/identical. One interpretation of this operation is that it is performing a kind of nearest neighbor computation (e.g., CITE), similar to memory retrieval in a dictionary, where similarity values between memory vectors (which are like templates stored in the dictionary) and query vectors are computed using the dot product. Then a weighted average of memory vectors is returned, where those memory vectors more similar to the query vector are given more weight. An SQHN network with a single hidden layer also performs a nearest neighbor operation. This simple network performs recall given query vector $x$ with the following equation
\begin{equation}\label{eq:SQHNSmall}
x_{new} = M argmax(\frac{1}{Z}M^T x),
\end{equation}
where $\frac{1}{Z}$ normalizes each elements of the hidden layer input, $M^T x$, such that the values range between zero and one and the input, x, and memory vectors, $m_j$, are normalized. Clearly, this is a kind of nearest neighbor operation similar to the MHN, except 1) the similarity is computed using a kind of cosine similarity (instead of dot product) and 2) instead of a weighted average, SQHN returns the single memory vector that is most similar to query vector. MHN may equivalently also use this hard recall procedure in the case where $\beta \rightarrow \infty$. However, special kinds of normalization would need to be added to the standard MHN and transformer attention to yield equivalence to SQHN.

\subsection{Supplementary Figures}
\begin{figure}[h]
\includegraphics[width=\textwidth]{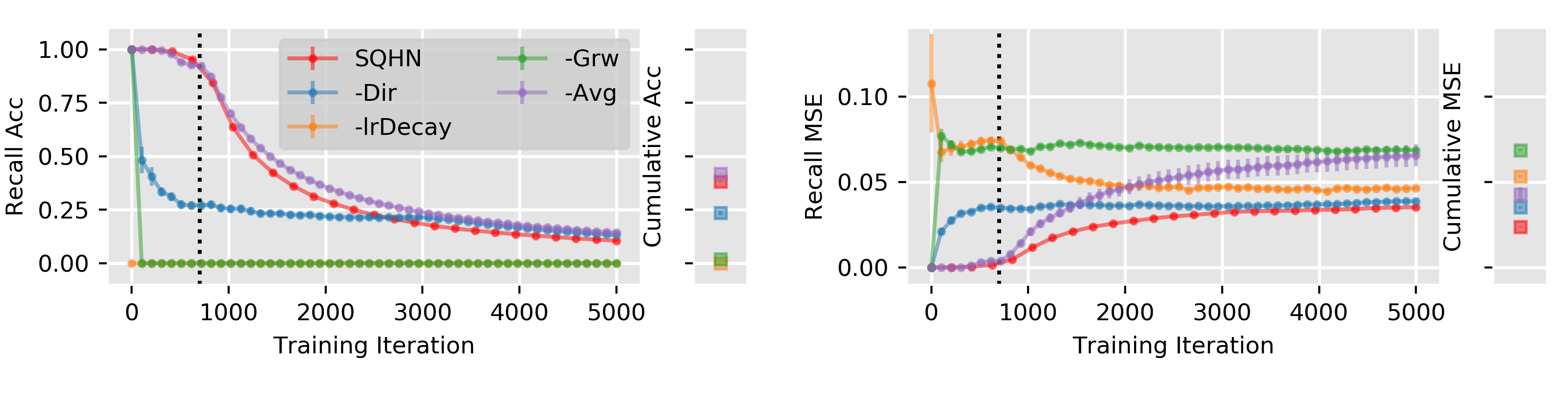}
\caption{Ablation study. We test the effects of removing various components of the SQHN learning algorithm in the online continual recall task on CIFAR-100 for the basic SQHN unit (figure \ref{fig:onlineContAbl}. We test the following ablations: 1) remove the Dirichlet based, exponentially decaying grow threshold and replace with a constant threshold (-Dir). 2) Remove the learning rate decay schedule and replace with a constant learning rate (-lrDecay). 3) Remove the grow operation and instead randomly initialize weights (-Grw). 4) Remove the averaging update and only grow new memory vectors (-Avg). The decaying grow threshold, learning rate schedule, and the grow operation are essential for high performance before the capacity (vertical dotted line) is reached. The grow operation, learning rate decay, and averaging operation are all essential for high performance after capacity is reached. Note that removing averaging helped slightly with the number of recalled images, however, doing so significantly worsened recall MSE. All components are necessary for a high cumulative performance under both measures.Recall accuracy and cumulative accuracy for various ablations is shown on the left. Recall MSE and cumulative MSE for the same ablations shown on right.}
\centering
\label{fig:onlineContAbl}
\end{figure}

\begin{figure}[h]
\includegraphics[width=.99\textwidth]{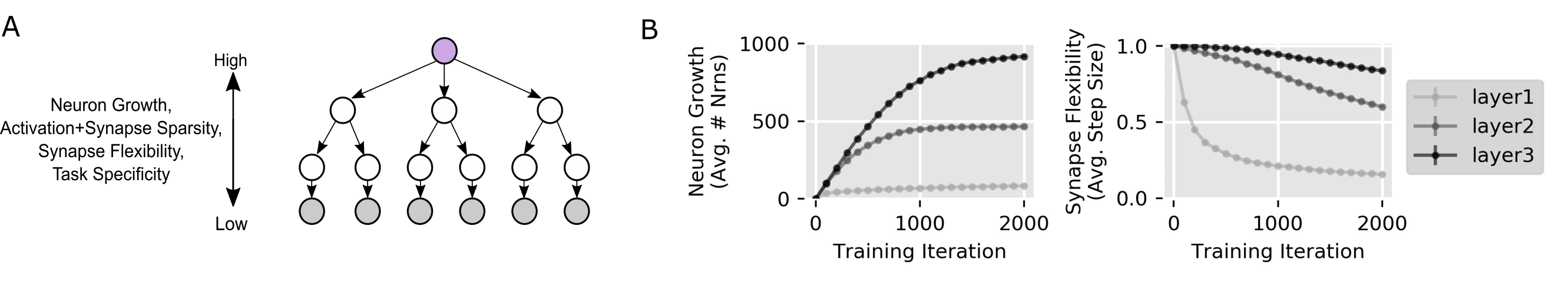}
\caption{Emergent properties. \textbf{A} Diagram summary of the emergent properties. \textbf{B} Measurements of neuron growth (left) and synapse flexibility (right) at each layer during training in online scenario in an SQHN with three hidden layers. Each node in the network has the same maximum number of neurons (1000), the same grow threshold, the same input kernal size (4x4), same learning rate decay, etc. The only difference between each node is its location in the network.}
\centering
\label{fig:emerge}
\end{figure}

\begin{figure}[h]
\includegraphics[width=.95\textwidth]{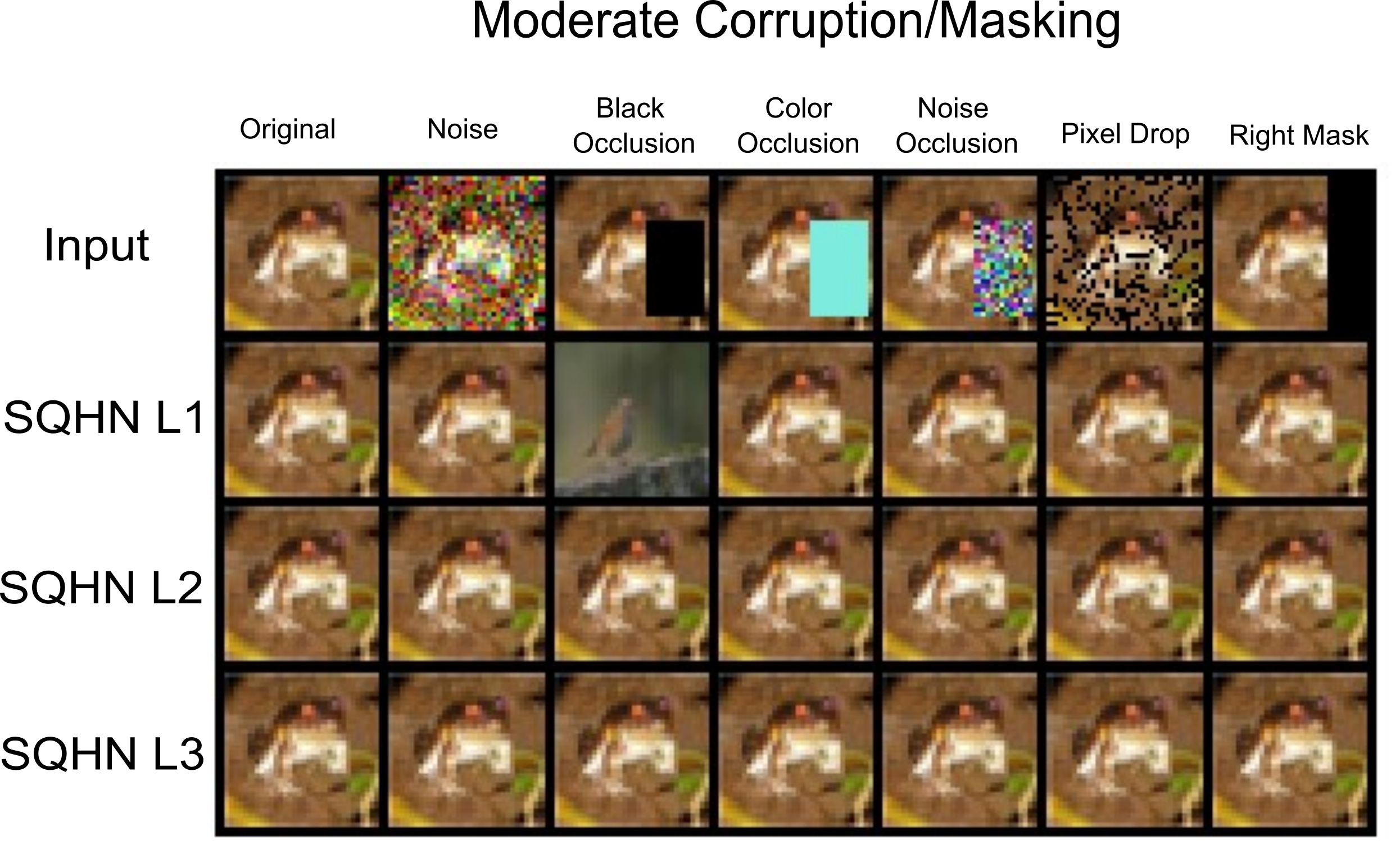}
\centering
\caption{Example outputs for moderate corruption scenario. Noise variance .2, masking and occlusion covers .25 of image. Trees are better at occlusion, but otherwise, with moderate corruption/masking all networks perform very well on all tasks.}
\label{fig:modCorruptEx}
\end{figure}

\begin{figure}[h]
\includegraphics[width=.95\textwidth]{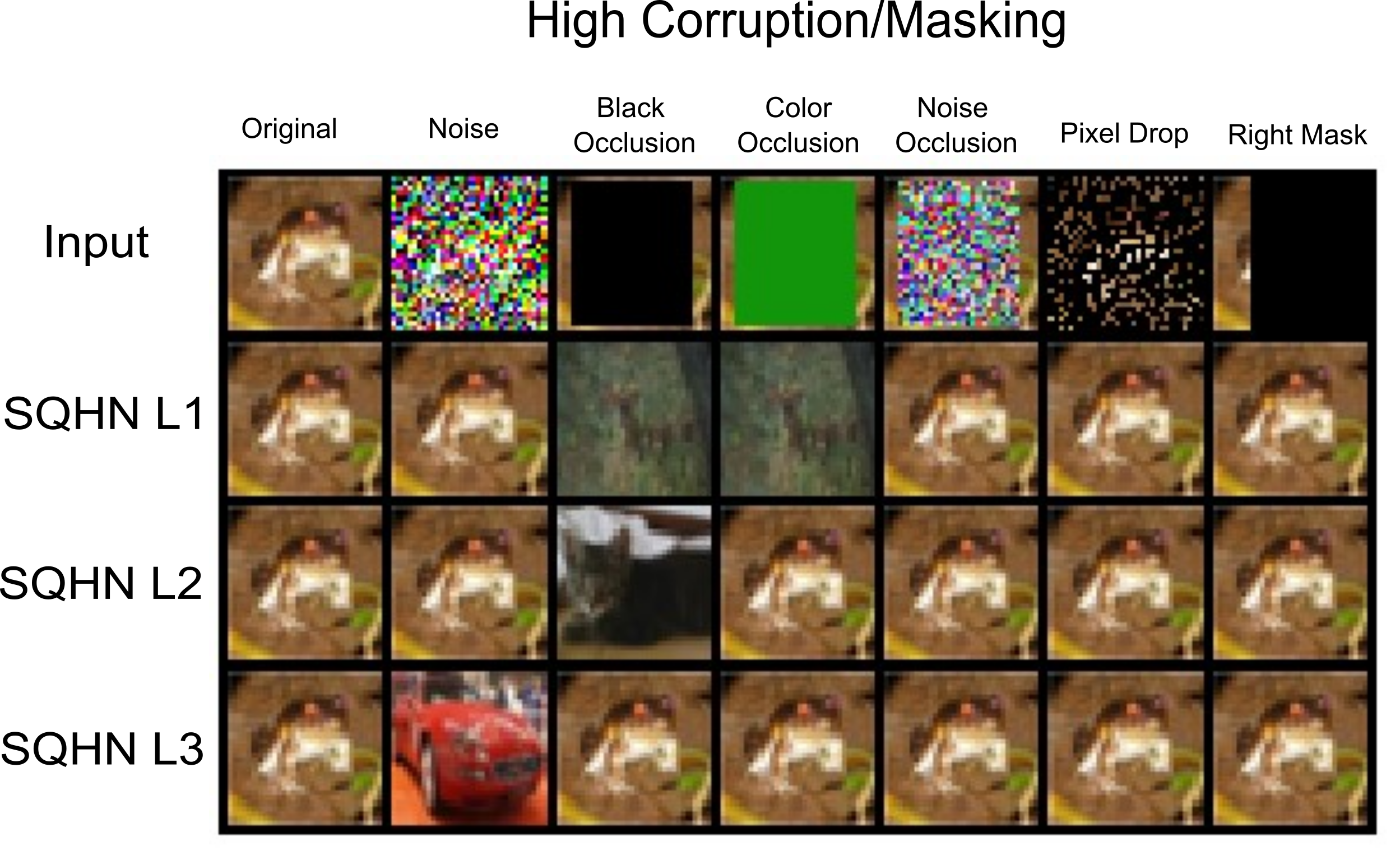}
\centering
\caption{Example outputs for high corruption scenario. Noise variance .8, masking and occlusion covers .75 of image. Taller trees, especially L3, are better at occlusion, but because they have small receptive fields at bottom layers they are more sensitive to noise. Smaller receptive fields mean smaller input vector per node, and smaller vector are more likely to overlap due to noise. The part-whole representation of the image, however, allows trees with multiple levels to better ignore occluded regions of the image.}
\label{fig:highCorruptEx}
\end{figure}


\end{document}